\documentclass{article}
\usepackage{microtype}
\usepackage{subfigure}
\usepackage{booktabs}
\usepackage{hyperref}

\usepackage{amsmath}
\usepackage{amssymb}
\usepackage{amsthm}
\usepackage{graphicx}
\usepackage{mathtools}
\usepackage{algorithmic}
\usepackage{algorithm}
\renewcommand{\algorithmicrequire}{\textbf{Input:}}
\renewcommand{\algorithmicensure}{\textbf{Output:}}

\DeclareMathOperator*{\argmax}{\mathrm{argmax}}
\newtheorem{thm}{Theorem}
\newtheorem{lem}[thm]{Lemma}

\newtheorem{ass}{Assumption}

\usepackage{color}

\usepackage{mathrsfs}

\newcommand{\nn}{\mathbb{N}}
\newcommand{\rr}{\mathbb{R}}

\newcommand{\ee}{\mathbb{E}}

\newcommand{\zz}{\mathbb{Z}}

\newcommand{\cn}{\mathcal{N}}

\newcommand{\cg}{\mathcal{G}}
\newcommand{\cp}{\mathcal{P}}
\newcommand{\cd}{\mathcal{D}}
\newcommand{\ct}{\mathcal{T}}
\newcommand{\ci}{\mathcal{I}}

\newcommand{\by}{\mathbf{y}}
\newcommand{\bff}{\mathbf{f}}

\newcommand{\matern}{Mat\'{e}rn}

\newcommand{\xnp}{x_{n+1}}

\newcommand{\tnp}{\tau_{n+1}}
\newcommand\numberthis{\addtocounter{equation}{1}\tag{\theequation}}

\newcommand{\romone}{I}
\newcommand{\romsec}{I\hspace{-.1em}I}
\newcommand{\romthr}{I\hspace{-.1em}I\hspace{-.1em}I}
\newif\ifisICML
\isICMLfalse 

\ifisICML
\usepackage{icml2020}
\else 
\usepackage{arxiv}
\usepackage[square, sort]{natbib}
\setcitestyle{numbers,square,citesep={,},aysep={,},yysep={,}}
\fi

\ifisICML
\icmltitlerunning{Time-varying Gaussian Process Bandit Optimization with Non-constant Evaluation Time}
\else
\fi

\begin{document}
\ifisICML
\else
\title{Time-varying Gaussian Process Bandit  Optimization with Non-constant Evaluation Time}

\makeatletter
\newcommand{\printfnsymbol}[1]{%
  \textsuperscript{\@fnsymbol{#1}}%
}
\makeatother

\author{
    \textbf{Hideaki Imamura} \textsuperscript{\rm{1, 2}},
    \textbf{Nontawat Charoenphakdee} \textsuperscript{\rm{1,2}},
    \textbf{Futoshi Futami} \textsuperscript{\rm{1, 2}},\\
    \textbf{Issei Sato} \textsuperscript{\rm{1, 2}},
    \textbf{Junya Honda} \textsuperscript{\rm{1, 2}},
    \textbf{Masashi Sugiyama} \textsuperscript{\rm{2, 1}} \\
    \textsuperscript{\rm{1}} The University of Tokyo, \textsuperscript{\rm{2}} RIKEN AIP\\
    \texttt{imamura@ms.k.u-tokyo.ac.jp}, 
    \texttt{nontawat@ms.k.u-tokyo.ac.jp}, 
    \texttt{futami@ms.k.u-tokyo.ac.jp}\\
    \texttt{sato@k.u-tokyo.ac.jp}, 
    \texttt{honda@k.u-tokyo.ac.jp},
    \texttt{sugi@k.u-tokyo.ac.jp} \\
}

\date{}
\maketitle
\fi

\begin{abstract}
  The Gaussian process bandit is a problem in which we want to find a maximizer of a black-box function with the minimum number of function evaluations. 
  If the black-box function varies with time, then time-varying Bayesian optimization is a promising framework. 
  However, a drawback with current methods is in the assumption that the evaluation time for every observation is constant, which can be unrealistic for many practical applications, e.g., recommender systems and environmental monitoring. 
  As a result, the performance of current methods can be degraded when this assumption is violated.
  To cope with this problem, we propose a novel time-varying Bayesian optimization algorithm that can effectively handle the non-constant evaluation time.
  Furthermore, we theoretically establish a regret bound of our algorithm.
  Our bound elucidates that a pattern of the evaluation time sequence can hugely affect the difficulty of the problem. 
  We also provide experimental results to validate the practical effectiveness of the proposed method.
\end{abstract}

\section{Introduction} \label{sec: intro}
Consider the problem of finding a maximizer of a black-box function with the minimum number of function evaluations.
Without making any assumptions on the objective function, this problem is known to be an ill-posed problem \citep{Srinivas_et_al_2010, Bogunovic_et_al_2016}.
A common assumption is to impose a smoothness on the objective function by introducing a Gaussian process (GP) \citep{Rasmussen_and_Williams_2006}.
Under this assumption, this problem is known as the GP bandit problem \citep{Mockus_et_al_1978}, which can be handled by an algorithm based on the Bayesian optimization (BO) framework \citep{Mockus_et_al_1978}.
Various types of settings and algorithms have been proposed
with theoretical and experimental validation
\citep{Srinivas_et_al_2010,
Krause_and_Ong_2011,
Henning_and_Schuler_2012,
Szegedy_et_al_2013,
Contal_et_al_2014,
Hernandez_Lobato_et_al_2014,
Bogunovic_et_al_2016,
Wang_and_Jegelka_2017,
Bogunovic_et_al_2018}.
There are many practical applications for this problem, e.g.,  recommender systems \citep{Vanchinathan_et_al_2014}, finance \citep{Hernandez_Lobato_et_al_2014}, environmental monitoring \citep{Srinivas_et_al_2010}, hyperparameter tuning \citep{Snoek_et_al_2012}, and robotics \citep{Lizotte_et_al_2007}.

However, in the real world, the objective function is often not static but varies with time.
For example, 
in recommender systems, the user preferences vary with trends \citep{Hariri_et_al_2015};
in the financial market, high growth-rate stocks change dynamically according to the economy \citep{Heaton_and_Lucas_1999};
and in environmental monitoring, observations in the environment change according to temperature and weather \citep{Bogunovic_et_al_2016}.
A recent study \citep{Bogunovic_et_al_2016} introduced {\it time-varying GP bandit optimization} to handle such changes of the objective function.
The algorithm proposed in \citet{Bogunovic_et_al_2016}
can automatically deal with the {\it forgetting-remembering trade-off}. 
More precisely, by modeling the change of the objective function based on a GP kernel with respect to time, it can handle the trade-off by forgetting outdated information while keep remembering the data that are still informative.

When we consider practical scenarios of the time-varying GP bandit, the time required for function evaluation often depends on the characteristic of the query point, and the uncertainty of the objective function increases as time.
For example,
in recommender systems, the evaluation time is the period between querying feedback from a user until the feedback is received, which may vary depending on the recommended products;
in finance, the evaluation time to determine profit depends on the type of bond;
in environmental monitoring, the evaluation time to monitor the environment depends on temperature and weather.
Existing work \citep{Bogunovic_et_al_2016} assumed that the
evaluation time is constant for all query points, which is difficult to satisfy in practice.
Furthermore, the increase in uncertainty for the objective function is assumed to be constant at each round.
For this reason, the previous study \citep{Bogunovic_et_al_2016} may fail to capture the dynamic behavior of the objective function and may not perform well, as we will show in Section \ref{sec: ex}. 

To overcome this limitation, in this paper, we consider a time-varying GP bandit problem with non-constant evaluation time.
We propose a bandit algorithm that can take the differences in evaluation time between query points into account by taking full advantage of the capability of GP kernels to model the continual change of the objective function.

If the evaluation time is non-constant, we can consider two types of goals in the time-varying setting: {\it maximization of reward per unit time} or {\it maximization of reward per action}.
The former corresponds to the case where we want to obtain high rewards in a short period of time, such as stock trading or advertisement optimization.
The latter corresponds to the case where the number of evaluations is limited due to the evaluation cost rather than evaluation time,
such as environmental monitoring. 
In this paper, we focus on the latter case where the evaluation cost dominates the evaluation time; thus, our goal is to maximize the reward per action.

{\bf Related Work:}
Many algorithms have been developed for the GP bandit problem in the BO framework and successfully used in various practical applications
\citep{Srinivas_et_al_2010,
Krause_and_Ong_2011,
Henning_and_Schuler_2012,
Szegedy_et_al_2013,
Contal_et_al_2014,
Hernandez_Lobato_et_al_2014,
Bogunovic_et_al_2016,
Wang_and_Jegelka_2017,
Bogunovic_et_al_2018}.
Although these algorithms assume the {\it time-invariant setting} where the objective function is static,
they tackle the critical challenge in the GP bandit problem, that is, the {\it exploration-exploitation trade-off}.
This means that in the GP bandit problem, we need to control the balance between collecting new data for improving the estimation of the objective function and choosing a promising point based on the data that have already been observed.
The GP upper confidence bound (GP-UCB) \citep{Srinivas_et_al_2010} is one of the most common algorithms that take the exploration-exploitation trade-off into account.
Note that many existing algorithms, as well as ours, are based on GP-UCB 
\citep{Krause_and_Ong_2011, Contal_et_al_2014, Bogunovic_et_al_2016, Bogunovic_et_al_2018}.

There have been few studies for the time-varying setting.
To the best of our knowledge, time-varying GP-UCB (TV-GP-UCB) \citep{Bogunovic_et_al_2016} is the only algorithm that can handle time-varying objective functions in the context of GP bandits.
On the other hand, several algorithms have been proposed in the context of multi-armed bandit problems with finitely many actions
\citep{Slivkins_and_Upfal_2008, Besbes_et_al_2014}.
Although they take into account the settings of finitely many arms, those algorithms are common in terms of considering the {\it forgetting-remembering trade-off}, i.e., to balance forgetting old information and remembering informative data.
In this paper, our proposed algorithm can also balance the forgetting-remembering trade-off by using a GP kernel 
to handle
the diminishing information.

In the context of BO, the non-constant evaluation time scenario has also been considered in
\citet{Swersky_et_al_2013} and \citet{Klein_et_al_2017}.
In their studies, the objective function minimized the evaluation time of all query points simultaneously by maximizing the objective function.
However, our algorithm does not minimize the total evaluation time but only focuses on maximizing the objective function by estimating the time-varying objective function.

As discussed in Section \ref{sec: alg}, our algorithm models
the objective function similarly to the {\it contextual GP bandit} 
\citep{Krause_and_Ong_2011,  Swersky_et_al_2013, Klein_et_al_2017}.
More precisely, we model the objective function as a sample from a GP using the product kernel over the context kernel and the action kernel.
We then construct the acquisition function to be maximized.
Note that unlike previous studies 
\citep{Krause_and_Ong_2011,  Swersky_et_al_2013, Klein_et_al_2017},
we do not assume that the context is given but we estimate it from data, which can be more realistic in real-world applications.

{\bf Contributions:}
We introduce a novel time-varying GP bandit algorithm in Section \ref{sec: alg}, which can take both the exploration-exploitation trade-off and forgetting-remembering trade-off into account.
We provide high-probability regret upper bounds for our algorithm in Section \ref{sec: th anal} and clarify that the difficulty arises from the deviation of the evaluation time.
Furthermore, we show that our regret bound covers two previous studies as special cases on the time-varying \citep{Bogunovic_et_al_2016} and time-invariant \citep{Srinivas_et_al_2010} settings.
Note that our analysis also covers a non-constant evaluation time scenario, unlike the existing settings.
We also investigate the experimental performance of our algorithm in comparison with existing algorithms in Section \ref{sec: ex} and confirm the practical superiority of our algorithm.

\section{Background} \label{sec: pre}
In this section, we formulate the GP bandit problem with a time-varying black-box function and introduce existing approaches.

\subsection{Problem Setting} \label{sec: prob set}

Let $\cd$ be an input domain of an objective function, which is a compact and convex subset of $\rr^d$.
Let $\ct = \rr_+$ be a time domain.
The objective function is denoted by $f: \cd \times \ct \rightarrow \rr$, where $f(x, \tau)$ represents the objective function value at a point $x$ and time $\tau$.

At each round $n \in \nn$, an agent can interact with $f$
only through querying an evaluation of a point $x_n \in \cd$.
The evaluation time for the $n$-th query is denoted by $t_n$ and the time point after the $n$-th evaluation is denoted by $\tau_n=\sum_{i=1}^n t_i$.
We assume that the time of querying the input $x_n$ can be ignored because it is dominated by the evaluation time.
At time $\tau_n$, we obtain a noisy evaluation $y_n = f(x_n , \tau_n) + z_n$, where $z_n,\,n=1,2,\dots,$ independently follow a Gaussian distribution $\cn(0, \sigma^2)$.
Let $\mathscr{D}_n = \{ (x_i, t_i, y_i) \}_{i=1}^n$ be the data obtained through $n$ observations.
At each round $n$, the agent chooses the next query point $x_{n+1}$ based on $\mathscr{D}_n$.

To measure the performance of algorithms in terms of reward per action, we use the notion of a \emph{regret} throughout this paper.
The regret for the $n$-th round is defined as
\begin{align}
    \label{formula: def of regret} r_n = \max_{x \in \cd} f(x, \tau_n) - f(x_n, \tau_n),
\end{align} 
which is the gap between the reward of the chosen point and the maximum reward at time $\tau_n$.
This notion of the regret is a natural extension of the regret $r_n = \max_{x \in \cd} f(x, n) - f(x_n, n)$ used in the previous algorithm \citep{Bogunovic_et_al_2016} since $\tau_n=n$ holds when evaluation is always performed in unit time, i.e., $t_i=1$ for $i=1,\ldots,n$.
We analyze our algorithm on the basis of the cumulative regret $R_n = \sum_{i=1}^n r_i$.
Note that even in the uniform setting, the length of unit time can also affect the performance. 
We illustrate this fact in the experiment section.

\subsection{Time-varying Gaussian Processes} \label{sec: tv gp}

We model the objective function as a sample from a GP \citep{Rasmussen_and_Williams_2006}, which is a common formulation for black-box optimization with a smoothness assumption \citep{Srinivas_et_al_2010, Bogunovic_et_al_2016, Krause_and_Ong_2011}.
As a result, the smoothness of the objective function is characterized by kernel functions.
Since the input of the objective function in our problem consists of two parts, time $\tau \in \ct$ and a point $x \in \cd$, we model both parts by kernel functions.
Let $k_{\mathrm{space}}: \cd \times \cd \rightarrow \rr_+$ be a space kernel and let $k_{\mathrm{time}}: \ct \times \ct \rightarrow \rr_+$ be a time kernel.
Let $k$ be a joint kernel function defined by $k = k_{\mathrm{space}} \otimes k_{\mathrm{time}}$.
We assume that the objective function $f$ is sampled from a GP with a mean function $\mu$ and a kernel function $k$ denoted by
 $\cg \cp (\mu, k)$.
Without loss of generality, we assume that $\mu = 0$ for GPs not conditioned on data \citep{Rasmussen_and_Williams_2006}.

In this paper, we focus on kernels that are stationary, i.e., shift-invariant.
Typical choices of such kernels are the exponential kernel $k_\mathrm{E}$, squared exponential kernel $k_{\mathrm{SE}}$, and Mat\'{e}rn kernel $k_{\mathrm{Mat\acute{e}rn}(\nu)}$. 
The hyperparameter $\nu$ in the Mat\'{e}rn kernel is called the smoothness parameter, since samples from a GP are
$\lfloor \nu - 1 \rfloor$-times
differentiable  \citep{Rasmussen_and_Williams_2006},
where $\lfloor x \rfloor$ means the largest integer not greater than $x$.
Note that both the exponential kernel and squared exponential kernel are the special cases of the {\matern} kernel when $\nu=1/2$ and $\nu \rightarrow \infty$, respectively.
In this paper, we will use the exponential kernel, squared exponential kernel, and Mat\'ern kernel with $\nu = 5/2$ for our experiments following the previous studies \citep{Krause_and_Ong_2011, Bogunovic_et_al_2016}.

The existing work \citep{Bogunovic_et_al_2016} proposed TV-GP-UCB, which uses the following special case of the exponential kernel for the time kernel function:
\begin{align}
    \label{formula: time kernel} k_{\mathrm{time}}(\tau, \tau') = (1 - \epsilon)^{\frac{|\tau - \tau'|}{2}},
\end{align}
where $\epsilon \in [0, 1]$ is a hyperparameter that controls the forgetting-remembering trade-off.
In \citet{Bogunovic_et_al_2016}, it was assumed that the evaluation time is identical for all query points, i.e., $t_n=1$ for all $n$.
By using the kernel $k_{\mathrm{time}}$,
the TV-GP-UCB is constructed as the time-varying GP model with 
$\tau_n = n$, and the value of the joint kernel is 
$
    k((x_i, \tau_i), (x_j, \tau_j)) 
    = k((x_i, i), (x_j, j))
    = k_{\mathrm{space}}(x_i, x_j) \times k_{\mathrm{time}}(i, j).
$


\subsection{Time-varying Bayesian Optimization} \label{sec: tv bo}

In the time-varying BO, an agent sequentially optimizes the time-varying objective function while balancing $\rm(\hspace{.18em}i\hspace{.18em})$ the exploration-exploitation trade-off and $\rm(\hspace{.08em}ii\hspace{.08em})$ the forgetting-remembering trade-off. 
Algorithm~\ref{fig: bo} illustrates a general framework of the time-varying BO \citep{Bogunovic_et_al_2016}. 
In the time-varying GP, it is required to specify how to select the next query point $x_{n+1}$ to be evaluated.
This selection procedure is determined by the acquisition function, which is designed to handle the trade-off between exploration of the search space and exploitation of current promising areas
\citep{Shahriari_et_al_2015}. 

Several previous studies \citep{Srinivas_et_al_2010, Krause_and_Ong_2011, Bogunovic_et_al_2016, Bogunovic_et_al_2018}
used the Upper Confidence Bound (UCB) acquisition function \citep{Auer_2002, Srinivas_et_al_2010} to balance the exploration-exploitation trade-off.
The UCB acquisition function is defined by
$
    \alpha_{\mathrm{UCB}}(x, \tau | \mathscr{D}_n) = \mu_n(x, \tau) + \beta_n \sigma_n(x, \tau),
$
where $\beta_n > 0$ is called the exploration weight.

In the previous study \citep{Bogunovic_et_al_2016},
the acquisition function is the special case of the above UCB function $\alpha_{\mathrm{UCB}}(x, \tau |  \mathscr{D}_n)$.
Since it is assumed that the evaluation time is always a unit time, i.e.,  $t_i = 1$,
$\tau_{n+1}$
is $n+1$ and the acquisition function can be expressed as
$
    \alpha_{\mathrm{UCB fixed}}(x | \mathscr{D}_n) 
    = \alpha_{\mathrm{UCB}}(x, n+1 | \mathscr{D}_n)
    = \mu_n(x, n+1) + \beta_n \sigma_n(x, n+1).
$
The algorithm based on this acquisition function is called TV-GP-UCB \citep{Bogunovic_et_al_2016}, which we refer to as TV hereafter.

\begin{algorithm}[t]
		\renewcommand{\algorithmicrequire}{\textbf{Input:}}
 		\renewcommand{\algorithmicensure}{\textbf{Output:}}
		\caption{Time-varying Bayesian Optimization}
 		\label{fig: bo}
 		\begin{algorithmic}[1]
		\REQUIRE An acquisition function $\alpha( \cdot | \mathscr{D}_n)$
		\STATE Initialize $\mathscr{D}_0 = \emptyset$.
		\FOR{$n = 0, 1, \ldots,$}
		\STATE Select a query point: $x_{n+1} = \argmax_{x \in \cd} \alpha(x | \mathscr{D}_n)$.
		\STATE Observe noisy evaluation: $y_{n+1} = f(x_{n+1}, \tau_{n+1}) + z_{n+1}$ and evaluation time: $t_{n+1} = \tau_{n+1} - \tau_n$.
		\STATE Update data: $\mathscr{D}_{n+1} = \mathscr{D}_n \cup \{ (x_{n+1}, t_{n+1}, y_{n+1}) \}$.
		\STATE Update a statistical model (such as GP).
		\ENDFOR
		\end{algorithmic}
\end{algorithm}

\section{Algorithms} \label{sec: alg}

In this section, we propose the Continuous Time-Varying GP-UCB algorithm (CTV), which can dynamically capture the time-varying objective function.

We consider two settings and propose three algorithms.
The first one is a simple setting where we know the evaluation time before evaluating the objective function at an input point.
The first algorithm, {\it CTV-fixed}, is designed for the first setting.
In the second setting, we do not know the evaluation time before evaluating the objective function at an input point of interest.
The second and third algorithms, {\it CTV} and {\it CTV-simple}, are designed for the second setting.
Here the CTV-simple algorithm is a computationally efficient version of the CTV algorithm.

First, we consider the setting in which we know the evaluation time
$t$ for each $x$ before evaluating $f(x)$.
The critical difference in our setting from the previous study is that the evaluation time depends on the chosen point.
Therefore, given an input point $x$, we can use the evaluation time $t$ to calculate the acquisition function $\alpha(x | \cd_n)$.

The acquisition function of the proposed algorithm, CTV-fixed, is as follows:
\begin{align}
    \label{formula: acq fixed} \alpha(x | \mathscr{D}_n)
    =
    \alpha_{\mathrm{base}}(x, \tau_n+t | \mathscr{D}_n),
\end{align}
where $\alpha_{\mathrm{base}}(x, \tau | \mathscr{D}_n)$ is an arbitrary acquisition function, which we call a base acquisition function.
Note that the evaluation time $t$ is determined by $x$ before calculating the value of \eqref{formula: acq fixed}.
The purpose of introducing CTV-fixed is to regard it as a gold standard for our problem 
that assumes the knowledge of evaluation time $t$.
Moreover, an analysis of CTV-fixed is insightful for understanding the characteristic of our problem, as shown in Section~\ref{sec: th anal}.

Second, we consider the setting where we do not know the evaluation time $t$ for each $x$ before evaluating $f(x)$.
To design a suitable algorithm for this setting, we construct an acquisition function that models the evaluation time as a sample from a new GP.
More specifically, let us assume that the evaluation time is expressed as a function of $x$.
To ensure the positivity, we parametrize the time as $t = \exp \left( g(x) \right)$, where $g \sim \cg \cp (\mu, k_{g})$ with a mean function $\mu$ and a kernel function $k_g$.
The observations $\{ (x_i, t_i)\}_{i=1}^n$ are expressed as $t_i = \exp (g(x_i) + \xi_i)$ with independent and identically distributed (i.i.d.) Gaussian noise $\xi_i \sim \cn (0, \sigma_g^2)$.

The posterior of $g$ is a GP with mean $\mu_n^g(x)$ and covariance $\kappa^g_n(x, x')$:
\begin{align*}
    \mu_n^g (x) &= \mu(x) + k^g_n(x)^T (K^g_n + \sigma^2_g I) ^ {-1} (\log \mathbf{t}_n - \boldsymbol{\mu}_n),\\
    \kappa^g_n(x, x') &= k_g(x, x') - k^g_n(x)^T (K^g_n + \sigma^2_g I) ^ {-1} k^g_n(x'),
\end{align*}
where 
$k^g_n(x) = (k_g(x_i, x))_{i=1}^n$, 
$\log \mathbf{t}_n = (\log t_i)_{i=1}^n$,
$\boldsymbol{\mu}_n = (\mu(x_i))_{i=1}^n$,
and $K^g_n = (k_g(x_i, x_j))_{i,j=1}^n$.
Therefore, the posterior distribution $p(t | x, \mathscr{D}_n)$ is a log-normal distribution 
$
p(t | x, \mathscr{D}_n) =  \Lambda (t | \mu_n^g(x), (\sigma^g_n)^2(x)),
$
where $(\sigma^g_n)^2(x) = \kappa^g_n(x, x)$ is the posterior variance.

The acquisition function of the proposed algorithm, CTV, is as follows:
\begin{align}
    \label{formula: acq integral} \alpha(x | \mathscr{D}_n) &= \int \alpha_{\mathrm{base}} \left(x, \tau_n + t | \mathscr{D}_n \right) p(t | x, \mathscr{D}_n) dt,
\end{align}
where $p(t | x, \mathscr{D}_n)$ is a posterior distribution for the function
$t=\exp (g(x))$.
The acquisition function in (\ref{formula: acq integral}) represents the posterior mean of the base acquisition function with respect to the evaluation time.

On the other hand, if we use mean prediction for $t$ instead of marginalization in \eqref{formula: acq integral}, we obtain the following acquisition function:
\begin{align}
    \label{formula: acq point} \alpha(x | \mathscr{D}_n) = \alpha_{\mathrm{base}}(x, \tau_n + \tilde{t} | \mathscr{D}_n),
\end{align}
where 
$
    \tilde{t} = \ee [t | x, \mathscr{D}_n] = \exp \left( \mu_n^g(x) + \frac{(\sigma^g_n)^2(x) + \sigma_g^2}{2} \right).
$
By using this acquisition function (\ref{formula: acq point}) instead of (\ref{formula: acq integral}), we can obtain a computationally more efficient algorithm since it does not need to compute the posterior expectation of the acquisition function.
We call this algorithm CTV-simple.

In our theoretical analysis and experimental validation, we use UCB acquisition function given by $\alpha_{\mathrm{base}}(x, \tau | \mathscr{D}_n) = \mu_n(x, \tau) + \beta_n \sigma_n(x, \tau)$
for an appropriately chosen $\beta_n$.

\section{Theoretical Analysis} \label{sec: th anal}
In this section, we establish the theoretical analysis to show the following two critical insights in the time-varying GP-bandit. 
First, our analysis links the existing analysis of the time-invariant setting~\cite{Srinivas_et_al_2010} and the time-varying setting~\cite{Bogunovic_et_al_2016}, while also generalizes beyond them. Second, our analysis shows that the pattern of evaluation time sequence can significantly affect the regret upper bound.
To show our insights, we analyzed the proposed method CTV-fixed for simplicity.

We do not assume any conditions on the expression of $t$, but we assume that we fix the sequence of the evaluation time $\{ t_i \}_{i=1}^n$ and we use the fixed evaluation time to calculate the value of the acquisition function at each round.
Note that the notation $\tilde{O}(\cdot)$ denotes the asymptotic growth rate up to logarithmic factors with respect to $n$.

In BO literature \citep{Srinivas_et_al_2010, Krause_and_Ong_2011, Bogunovic_et_al_2016}, the key quantity of theoretical analysis for UCB-based algorithms is the {\it maximum information gain} 
\citep{Cover_and_Thomas_1991}, which is defined 
in the following way:
Given sets of $\{ x_i \}_{i=1}^n$ and $\{ \tau_i \}_{i=1}^n$,
let $\bff_n$ and $\by_n$ be random vectors
$
    (f(x_1, \tau_1), \ldots, f(x_n, \tau_n))
$ and
$
(f(x_1, \tau_1) + z_1, \ldots, f(x_n, \tau_n) + z_n)
$
respectively.
For these random vectors, the informativeness of sampled points $x_1, \ldots, x_n$ on $f$ is measured by the {\it information gain}, which is the mutual information 
$
\tilde{I}(\by_n; \bff_n) = H(\by_n) - H(\by_n | \bff_n).
$
Here, $H(X)$ (resp.~$H(X|Y)$) denotes the differential entropy of $X$
(resp.~conditional differential entropy of $X$ given $Y$).
Then, the {\it maximum information gain} is defined by
\begin{align}
    \label{formula: maximum information gain} \tilde{\gamma}_n = \max_{x_1, \ldots, x_n} \tilde{I}(\by_n; \bff_n).
\end{align}

Unlike the time-invariant setting,
the key quantity of our analysis is not only the maximum information gain but also the {\it maximum space information gain}, which is the maximum information gain that 
only takes into account the space information.
Note that the differential entropy for a Gaussian distribution can be expressed as $H(N(\mu, \Sigma)) = \frac{1}{2} \log \det 2 \pi e \Sigma$.
Since we assume that the objective function is a GP, the information gain satisfies
$\tilde{I}(\by_n; \bff_n) = \frac{1}{2} \log \det (I + \sigma^{-2} \tilde{K}_n)$,
where $I$ is the $n \times n$ identity matrix,
$\sigma^2$ is the noise variance of observations,
and $\tilde{K}_n$ is the Gram matrix defined by
$\tilde{K}_n = (k((x_i, \tau_i), (x_j, \tau_j)))_{i,j=1}^n$.
By this identity, for the space kernel, we define the maximum space information gain by $\gamma_n = \max_{x_1, \ldots, x_n} \frac{1}{2} \log \det (I + \sigma^{-2} K_n) $, 
where $K_n = (k_{\mathrm{space}}(x_i, x_j))_{i,j=1}^n$.

Another key quantity of the theoretical analysis of our algorithms is the
{\it evaluation time uniformity}, which is defined by 
$
C_{\epsilon, \ct'} = \sum_{\tau_j \in \ct'} \sum_{\tau_k \in \ct'} \min \left(\frac{1}{\epsilon^2}, (\tau_j - \tau_k)^2 \right)
$
for any $\epsilon > 0$ and the finite subset $\ct' \subset \ct$.
To investigate the dependence of our regret upper bound on $C_{\epsilon, \ct'}$,
we consider two examples
of the evaluation time: the {\it uniform setting} and the {\it extremely biased setting}.

Let $T = \tau_n = \sum_{i=1}^n t_i$.
The uniform setting is the case where the evaluation time is fixed to $t_i = \frac{T}{n}$,
which leads to $\tau_i = \frac{T}{n}i$ for any $i \in [n]$.
By this definition, we can analytically calculate the evaluation time uniformity.
The proof of the following lemma is given in Appendix \ref{sec: proof of lem}.
\begin{lem} \label{lem: uni uni}
    Consider the uniform setting of evaluation time.
    Pick any $i \in [n]$.
    Let 
    $
        \ct' = \{ \tau_{k_0 + 1}, \ldots, \tau_{k_0 + i} \}
    $
    for some $k_0 \le n - i$,
    that is, 
    $\ct'$ be a subset of $\{ \tau_k \}_{k=1}^n$ with $i$ consecutive elements.
    Then, the evaluation time uniformity is
    $
        C_{\epsilon, \ct'}
        = 
        \frac{1}{6} \frac{T^2}{n^2} i^2 (i^2 - 1)
    $
    if $i \le \frac{n}{\epsilon T}$,
    and
    $
    C_{\epsilon, \ct'}
        =
        \frac{T}{\epsilon n}
        \left(
            \frac{1}{2} \left( \frac{n}{T \epsilon} \right)^3 
            - \frac{4}{3} i \left( \frac{n}{T \epsilon} \right)^2
        + \left( i^2 - \frac{1}{2} \right)\frac{n}{T \epsilon}
            + \frac{i}{3} \right)
    $
    if $i \ge \frac{n}{\epsilon T}$.
\end{lem}

On the other hand, the extremely biased setting is the case where $t_{n_0} = T$ for some $n_0 \in [n]$ and $t_i = 0$ for any $i \neq n_0$, 
which leads to $\tau_i = 0$ for any $i < n_0$ and $\tau_i = T$ for any $i \ge n_0$.
By this definition, we obtain the following lemma, whose proof can be also found in Appendix \ref{sec: proof of lem}.
\begin{lem} \label{lem: bia uni}
    Consider the extremely biased setting of evaluation time.
    Let 
    $
        \ct' = \{ \tau_{k_0 + 1}, \ldots, \tau_{k_0 + i} \}
    $
    for some $k_0 \le n - i$.
    Then, the evaluation time uniformity is
    $
        C_{\epsilon, \ct'} = 0
    $
    if $\tau_{n_0} \not \in \ct'$,
    and
    $
    C_{\epsilon, \ct'} = 
        2 ( n_0 - k_0 - 1) (k_0 + i - n_0 + 1)
        \min \left( \frac{1}{\epsilon^2}, T^2\right)
    $
    if $\tau_{n_0} \in \ct'$.
\end{lem}

The uniform setting is similar to
that of
\citet{Bogunovic_et_al_2016}.
On the other hand, the extremely biased setting is the extreme case of the evaluation time.
We expect that the setting of the evaluation time in the real-world is a mixture of the uniform setting and the extremely biased setting.

Recall that the regret is defined as \eqref{formula: def of regret}. 
We make the following smoothness assumption to derive the regret bound of the proposed algorithm.
\begin{ass} \label{ass}
    (\romone) The generated sample path $f$ from a GP is almost surely continuously differentiable.
    
    (\romsec) The time kernel satisfies $1 - k_{\mathrm{time}}(\tau, \tau') \le \epsilon |\tau - \tau'|$ 
    for some $\epsilon \in [0, 1]$ and for any $\tau, \tau' \in \ct$.
    
    (\romthr) The joint kernel satisfies
    $\forall L \ge 0, \tau \in \ct, j \in [d]$
    \begin{align*}
        \Pr \left(\sup_{x \in \cd} \left| \frac{\partial f(x, \tau)}{\partial x^{(j)}}\right| > L \right) \le a \exp \left( - \frac{L^2}{b^2} \right)
    \end{align*}
    for some $a, b > 0$.
\end{ass}
The following theorem gives a regret upper bound for the cumulative regret $R_n = \sum_{i=1}^n r_i$.
\begin{thm} \label{thm: main}
    Let the domain $\cd$ be a subset of  $[0, r]^d$, 
    and suppose that the space kernel and the time kernel satisfy Assumption \ref{ass}.
    Pick $\delta \in (0, 1)$ and set
    \begin{align*}
        \beta_n = 2 \log \frac{2 \pi^2 n^2}{3 \delta} + 2d \log \left( d n^2 b r \sqrt{\log \frac{2 \pi^2 n^2 a d }{3 \delta}} \right).
    \end{align*}
    Let 
    $C = 8 / \log (1 + \sigma^{-2})$.
    Let
    $
        \ct_{i} = \{ \tau_{d_{i-1}+1}, \ldots, \tau_{d_i} \} \subset \{ \tau_k \}_{k=1}^n
    $
    for some
    $
        0 = d_0 < d_1 < d_2 < \cdots < d_{N-1} < d_N = n.
    $
    Set $M_i = d_i - d_{i-1}$ and $M = \max_{i=1, \ldots, N} M_i$.
    Then, for any $n \ge 0$, our algorithm CTV-fixed satisfies
    \begin{align*}
        &R_n
        \le \sqrt{C \beta_n n \tilde{\gamma}_n} + 2 \\
        &\le \sqrt{C \beta_n n 
       \left( N \gamma_M 
        +
        \frac{1}{2} \sum_{i = 1}^{N} M_i \phi \left( \sigma^{-2} \epsilon \sqrt{\frac{C_{\epsilon, \ct_{i}}}{M_i}}
        \right)
        \right) } + 2 \numberthis \label{formula: reg upper bound}
    \end{align*}
    with probability greater than $1 - \delta$,
    where
    $\phi(x) = \min \left( x, \log x + \frac{1}{x}\right)$.
\end{thm}
A proof of this theorem is given in Appendix \ref{sec: proof of main thm}.
If the evaluation time is fixed to $t_i = 1$ for all $i \in [n]$,
then the order of this result with respect to $n$ is equivalent to the regret bound of \citet{Bogunovic_et_al_2016}: $\tilde{O}(n)$.
On the other hand, if there is
no time-varying effect, that is, $\epsilon = 0$ or $C_{\epsilon, \ct'} = 0$,
then this upper bound is equivalent to that of the existing work which does not consider the time-varying objective function \citep{Srinivas_et_al_2010}.
Our theoretical bound is the first generalized analysis linking these two studies, to the best of our knowledge.

In the regret bound in \eqref{formula: reg upper bound}, we can arbitrarily choose the partition $\{d_i\}$ to minimize the RHS.
By 
appropriately choosing $\{d_i\}$ for each kernel, we obtain the following theorem.
\begin{thm} \label{thm: derived}
    Let $T = \tau_n = \sum_{i=1}^n t_i$.
    Suppose that the time kernel is
    given by \eqref{formula: time kernel}.

        (\romone) Suppose the evaluation time $\{ t_i \}_{i=1}^n$ is
            uniform, that is, $t_i = \frac{T}{n}$.
            (a) Suppose the space kernel is the squared exponential kernel.
            Then,
            with high probability,
            \begin{align*}
                R_n =
                \begin{cases}
                    \tilde{O} (\sqrt{n})
                    & \mbox{if } \epsilon T < n^{- \frac{3}{2}},\\
                    \tilde{O} \left( n^\frac{4}{5} T^\frac{1}{5} \epsilon^\frac{1}{5} \right)
                    & \mbox{if } n^{- \frac{3}{2}} \le \epsilon T \le n,\\
                    \tilde{O} \left(n 
                    \left(1
                    +
                    \left(
                        \frac{\epsilon T}{n}
                    \right)^\frac{1}{2}
                    \right)
                    \right)
                    & \mbox{if } n \le \epsilon T.
                \end{cases}
            \end{align*}
            (b) Suppose the space kernel is the Mat\'{e}rn kernel with parameter $\nu$
            and $c = \frac{d(d+1)}{2 \nu + d(d + 1)}$.
            Then, 
            with high probability,
            \begin{align*}
                R_n =
                \begin{cases}
                    \tilde{O} (\sqrt{n^{1+c}})
                    & \mbox{if } \epsilon T < n^{- \frac{3}{2}},\\
                    \tilde{O} \left( n^\frac{4-c}{5-2c} T^\frac{1-c}{5-2c} \epsilon^\frac{1-c}{5-2c} \right)
                    & \mbox{if } n^{- \frac{3}{2}} \le \epsilon T \le n,\\
                    \tilde{O} \left(n 
                    \left(1
                    +
                    \left(
                        \frac{\epsilon T}{n}
                    \right)^\frac{1}{2}
                    \right)
                    \right)
                    & \mbox{if } n \le \epsilon T.
                \end{cases}
            \end{align*}
        \\
        (\romsec) Suppose the evaluation time
            $\{ t_i \}_{i=1}^n$ is
            extremely biased, that is,
            $t_i = 0$ when $i \neq n_0$ and $t_{n_0} = T$.
            (a) Suppose the space kernel is the squared exponential kernel. Then,
            with high probability,
            $
            R_n = \tilde{O} 
            \left(
                \sqrt{
                    n
                }
            \right).
            $
            (b) Suppose the space kernel is the Mat\'{e}rn kernel with parameter $\nu$ and $c = \frac{d(d+1)}{2 \nu + d(d + 1)}$.
            Then,
            with high probability,
            $
            R_n = \tilde{O} 
            \left(
                \sqrt{
                    n^{1+c}
                }
            \right).
            $
\end{thm}
A proof of this theorem is given in Appendix \ref{sec: proof of derived thm}.
To the best of our knowledge, this result is the first asymptotic regret analysis capturing the dependence
not only on the number of iterations $n$ but also
on $T$.

Theorem \ref{thm: derived} gives results for two settings of the variation of $\{ t_i \}_{i=1}^n$, that is, the {\it uniform} and {\it extremely biased} settings.
We discuss the results for these two settings in more detail as follows.

The uniform setting with $T = n$ is equivalent to the setting of \citet{Bogunovic_et_al_2016}.
This setting leads to the regret bound of $\tilde{O}(n)$ for both the squared exponential kernel and Mat\'{e}rn kernel \citep{Bogunovic_et_al_2016}.
However, in terms of $\epsilon$, our regret achieves $\tilde {O}(\epsilon^{\frac{1}{5}})$ for the squared exponential kernel and $\tilde{O} \left( \epsilon^\frac{1-c}{5-2c}\right)$ for the Mat\'{e}rn kernel with $ n^{- \frac{3}{2}} \le \epsilon \le n$.
This regret strictly improves that of \citet{Bogunovic_et_al_2016}: $\tilde{O} \left(\epsilon^{\frac{1}{6}} \right)$ for the squared exponential kernel and $\tilde{O} \left(\epsilon^\frac{1-c}{6-2c}\right)$ for the Mat\'{e}rn kernel.

In the uniform setting with $T = n$, if we fix $n$ and take a limit of $\epsilon \rightarrow 0$,
the regret bound becomes 
$\tilde{O}(\sqrt{n})$ for the squared exponential kernel 
and $\tilde{O}(\sqrt{n^{1+c}})$ for the {\matern} kernel.
This bound matches
that of the time-invariant setting \citep{Srinivas_et_al_2010}.
On the other hand, if we fix $\epsilon$ and take a limit of $n \rightarrow \infty$,
the regret bound becomes 
$\tilde{O}(n)$ for both kernels, which match that of the existing time-varying setting \citep{Bogunovic_et_al_2016}.
Our result links the time-invariant setting and the time-varying setting and 
shows
the effect of $n$ and $1/\epsilon$ on the regret upper bound.

In the extremely biased setting, the regret bounds are the same as those of the time-invariant setting \citep{Srinivas_et_al_2010}.
This regret upper bound is totally different from that of the uniform setting because we can obtain a \emph{sublinear regret}.
Our result shows that even in the time-varying GP bandit problem, the sublinear regret can be achieved by our algorithm depending on the sequence of evaluation time.

Our results show that the order of the regret bound significantly changes depending on the sequence of evaluation time.
One may argue that both the uniform and extremely biased settings are unrealistic. 
We emphasize that the purpose of our analysis is to illustrate that in the time-varying GP bandit with non-constant evaluation time, the sequence of non-constant evaluation time can significantly affect the difficulty of the problem, which is theoretically justified from our analysis.

\section{Experiments} \label{sec: ex}

In this section, we report the experimental results that compared our algorithms CTV and its simplified version CTV-simple with commonly
used
existing algorithms, which are GP-UCB \citep{Srinivas_et_al_2010} and
the existing algorithm designed for the time-varying setting,
TV-GP-UCB (TV) \citep{Bogunovic_et_al_2016}.
We used the CTV-fixed algorithm as a golden standard method in experiments, since it is assumed that the CTV-fixed algorithm knows the all evaluation time $t$ for any $x$ beforehand.
In \citet{Bogunovic_et_al_2016},
two types of the TV-GP-UCB algorithms were proposed: Resetting-GP-UCB and Time-varying-GP-UCB.
Since it is known that Resetting-GP-UCB performs
poorly compared with
Time-varying-GP-UCB \citep{Bogunovic_et_al_2016}, we did not include Resetting-GP-UCB in our experiment.


We used the {\matern(5/2)} kernel as the GP kernel of GP-UCB and for the space kernel of TV, CTV, and CTV-simple.
We used the special case of the exponential kernel
in
\eqref{formula: time kernel}
for the time kernel of TV, CTV and CTV-simple.
We used the {\matern(5/2)} kernel for a GP modeling of $t = \exp (g(x))$ in CTV and CTV-simple.

We use the same $\beta_n$ as in the previous study \cite{Bogunovic_et_al_2016}.
We used the Limited-memory Broyden-Fletcher-Goldfarb-Shanno algorithm (L-BFGS) \citep{Liu_1989} for maximization of the acquisition function in all algorithms.

\subsection{Data Description}
We used a two-dimensional input domain $\cd = [0,1]^2$ and quantized it into $50 \times 50$ uniformly divided points.
We generated the time-varying objective function according to the following time-varying GP model.
The time-varying objective function $f_i$ after $i$ seconds is according to
$f_0 \sim \cg \cp (0, k)$ and
$f_{i+1} = \sqrt{1 - \lambda} f_{i} + \sqrt{\lambda} \eta$, where $\lambda = 0.01$,
$\eta \sim \cg \cp (0, k)$ and $k$ is a kernel function such as the squared exponential kernel or {\matern}($5/2$) kernel.
We set $l = 0.2$ and $\theta = 1.0$ for the kernel parameters.
All sampling noises were
set to 0.01 to achieve $1\%$ of signal variance.

We considered two
settings of evaluation time, which are
the uniform setting and the biased setting.
Note that the uniform setting 
coincides with the setting in
\citet{Bogunovic_et_al_2016}.
In the uniform setting, the evaluation time at any point is fixed to 3.
For the biased setting, which is novel, the evaluation depends on the points.
We set
the evaluation time $t$ as
$
    t(x) = 2(\sin(\sqrt{2} \pi \|x\|_2) + 2).
$

\subsection{Experimental Design}
We conducted two experiments to validate our algorithms.
The hyperparameters of all algorithms were carefully chosen to maximize their performance using cross-validation or manually selected.

In the first
experiment,
we validated our algorithms in terms of the 
cumulative regret per round $R_n/n$.
There were four experimental settings for the objective function generated from two kernels and two settings of the evaluation time.
For each experimental setting, we ran the five GP BO algorithms, i.e., GP-UCB, TV, CTV, CTV-simple, and CTV-fixed in $100$ iterations.
We conducted this experiment 30 times.
Figure \ref{fig: avg reg} is a plot between the number of iterations and the cumulative regret per round $R_n/n$ with the mean and the standard deviation.
Note that since we randomly chose 30 points for the initialization, the horizontal axis of those graphs starts from $30$.

In the second experiment, we compared proposed and existing algorithms with different length of the evaluation time $t$ in the uniform setting by the squared exponential kernel.
We set $t$ to $1$ and $10$.
If we set $t = 1$, the evaluation time is exactly the same as that of the existing study
\citep{Bogunovic_et_al_2016}.
Therefore, TV is expected to perform best in this setting.
On the other hand, if we set $t = 10$, the evaluation time could be too long that all methods may fail.
The result is shown in Figure \ref{fig: step}.

\begin{figure*}[t]
    \centering
    \begin{tabular}{c}
        \begin{minipage}{0.4\hsize}
            \centering
            \includegraphics[width = 0.9\columnwidth]{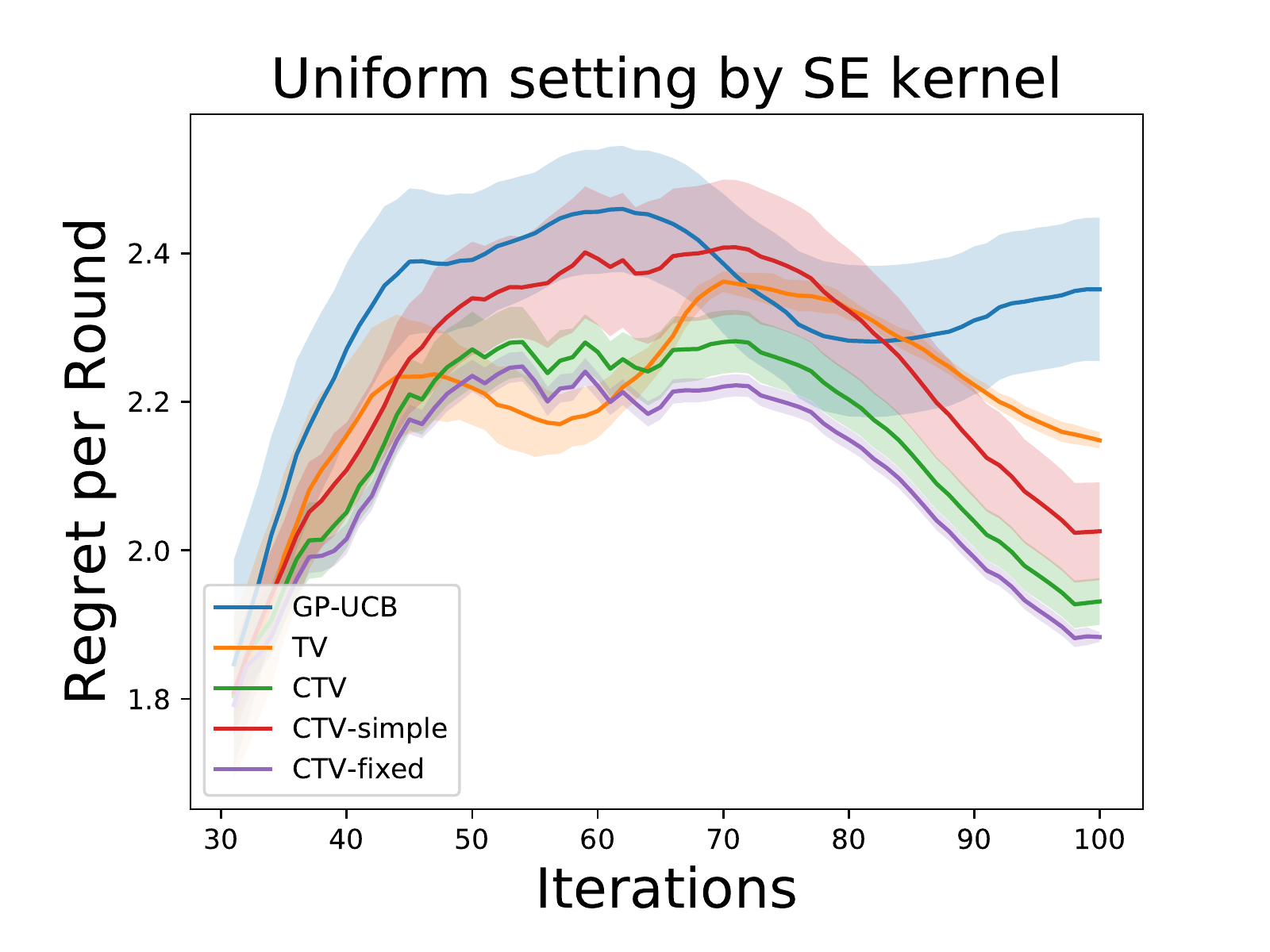}
        \end{minipage}
        \begin{minipage}{0.4\hsize}
        \centering
            \includegraphics[width = 0.9\columnwidth]{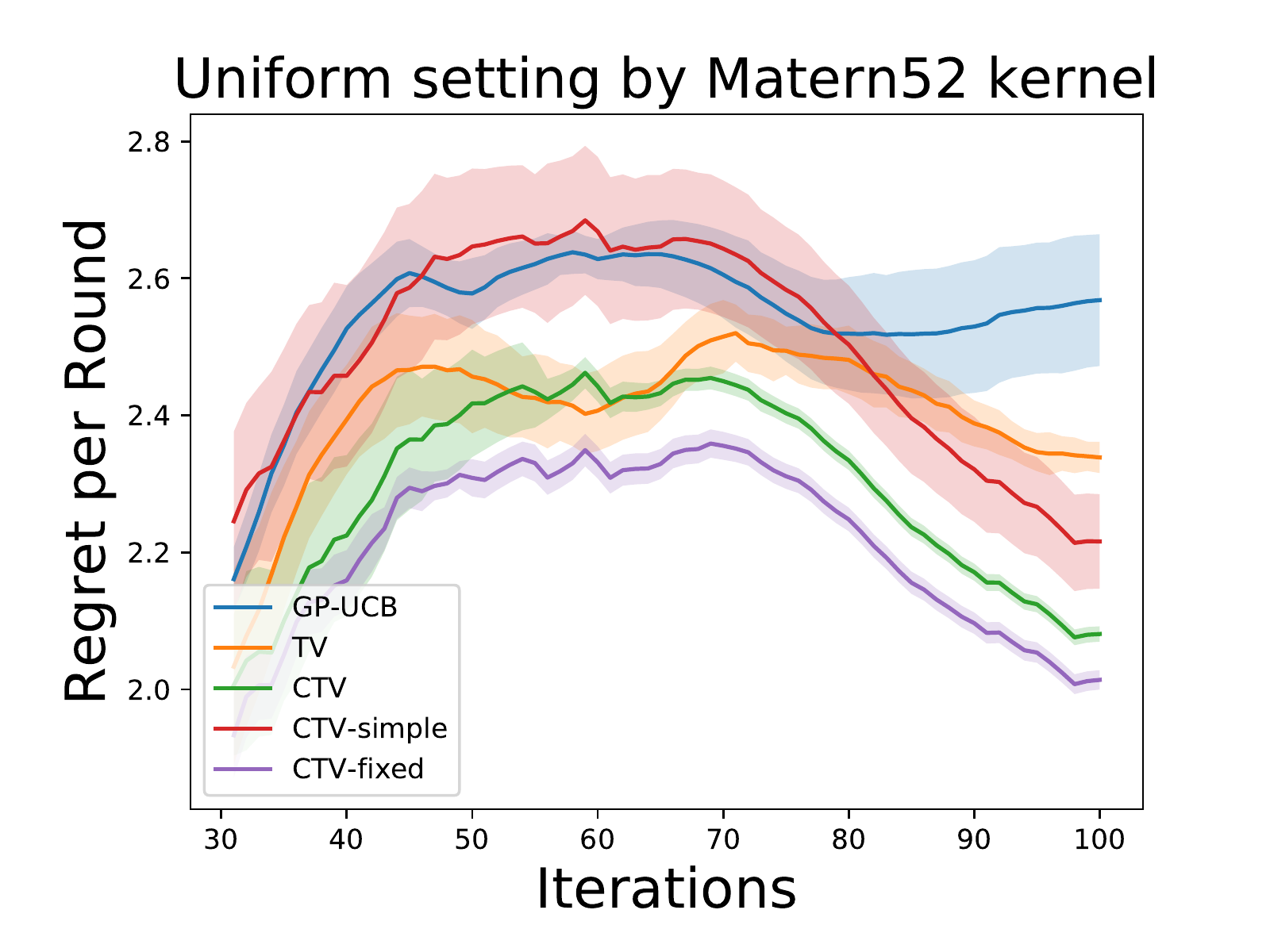}
        \end{minipage} \\
        \begin{minipage}{0.4\hsize}
            \centering
            \includegraphics[width = 0.9\columnwidth]{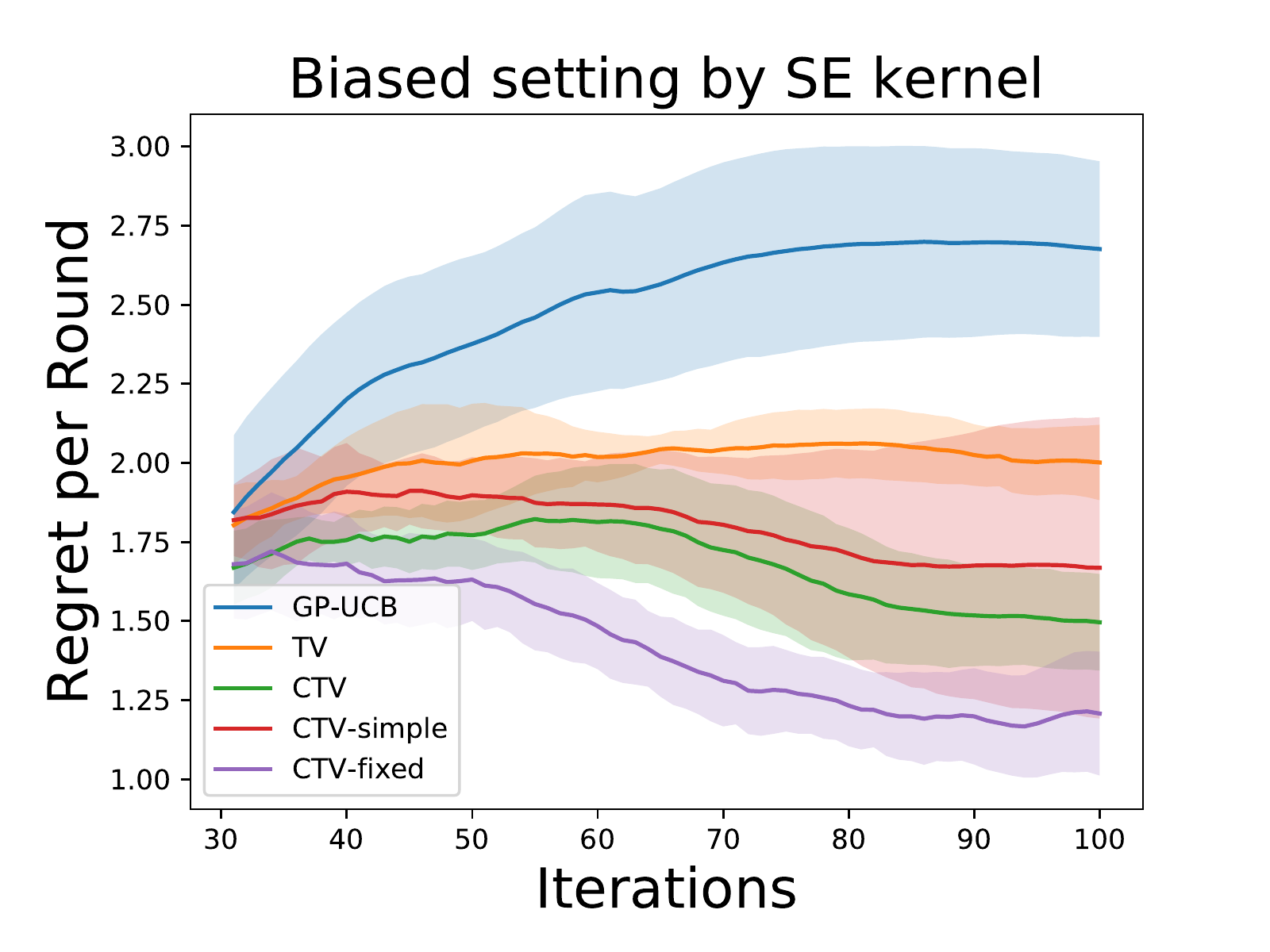}
        \end{minipage}
        \begin{minipage}{0.4\hsize}
        \centering
            \includegraphics[width = 0.9\columnwidth]{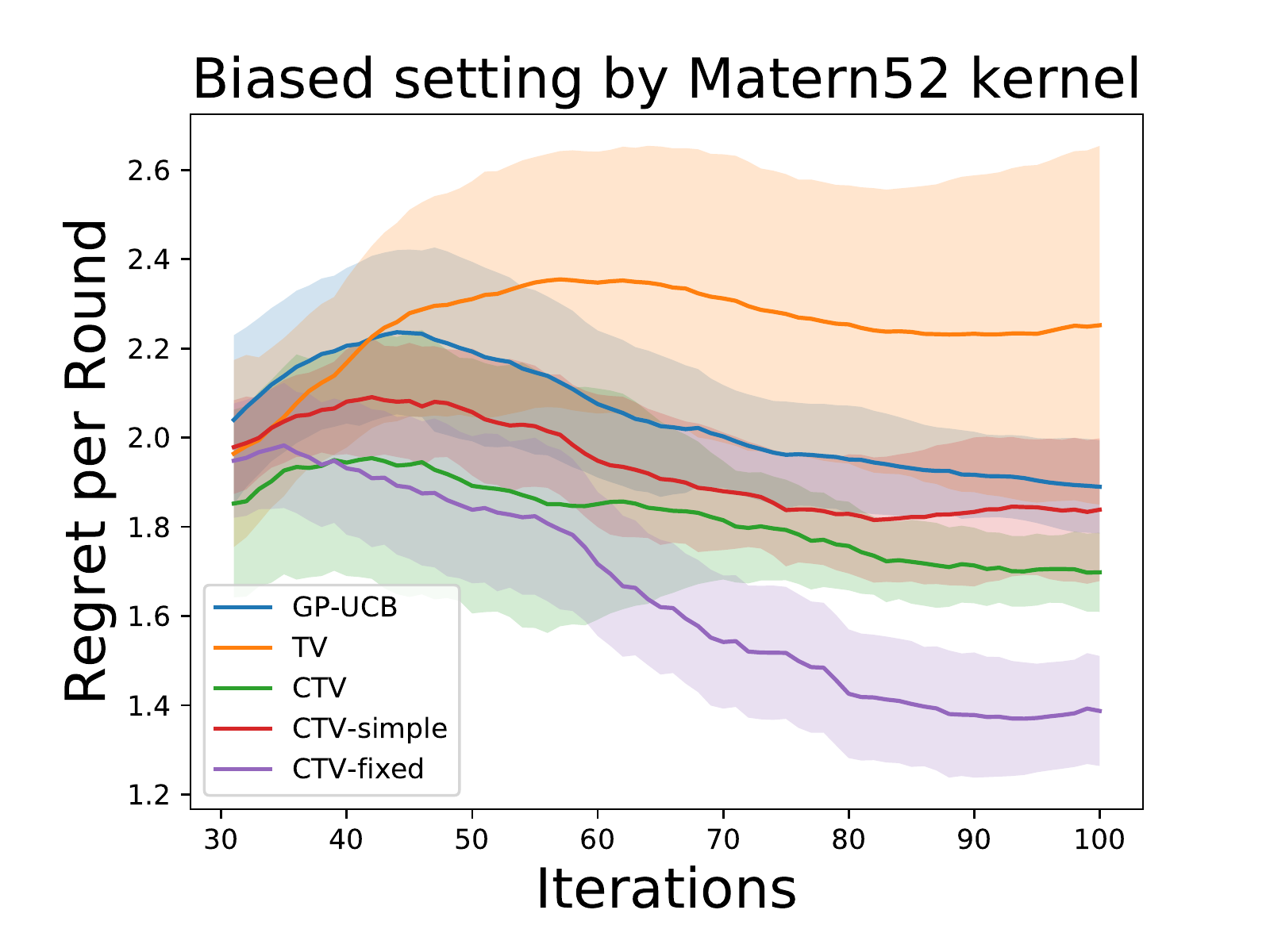}
        \end{minipage}
    \end{tabular}
    \caption{Averaged cumulative regret for the squared exponential and {\matern} (5/2) kernel in the uniform and biased settings.}
    \label{fig: avg reg}
\end{figure*}
\begin{figure*}[t]
    \centering
    \small{
    \begin{tabular}{c}
    \begin{minipage}{0.4\hsize}
    \centering
        \includegraphics[width = 0.9\columnwidth]{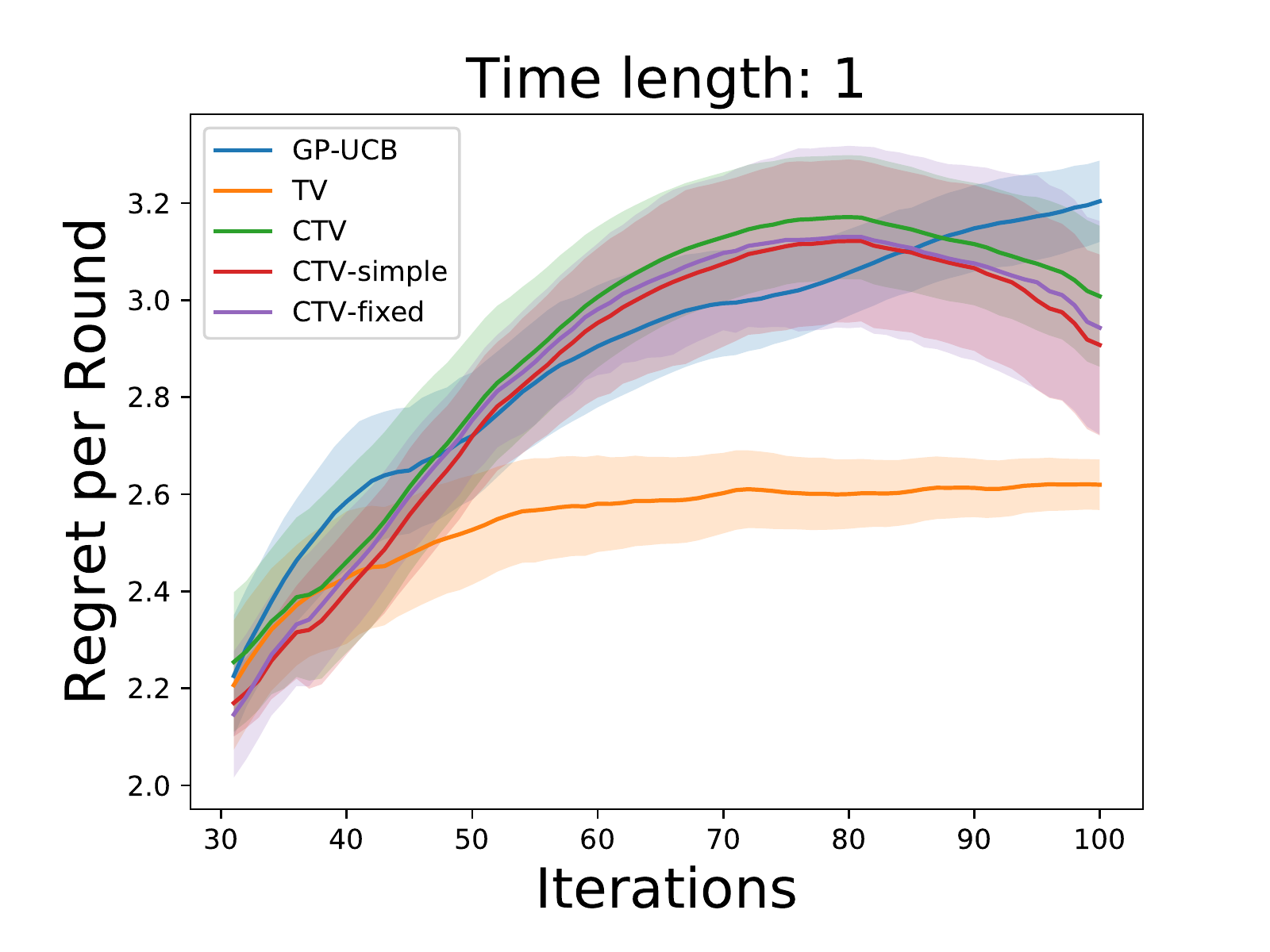}
    \end{minipage}
    \begin{minipage}{0.4\hsize}
        \centering
        \includegraphics[width = 0.9\columnwidth]{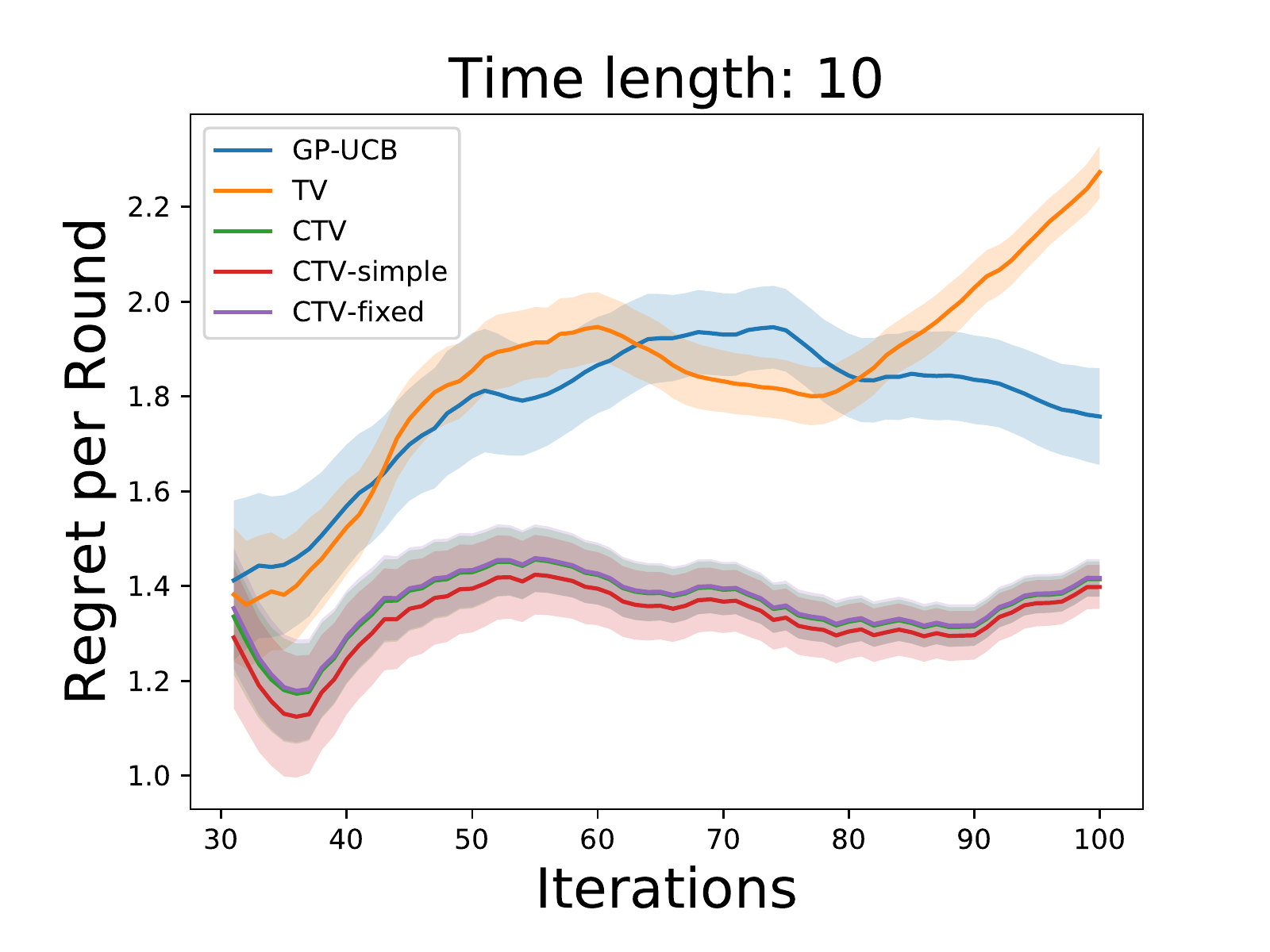}
    \end{minipage}
    \end{tabular}
    }
    \caption{Average cumulative regret for the squared exponential with the different length of evaluation time in the uniform setting.}
    \label{fig: step}
\end{figure*}

\subsection{Discussion}
Figure \ref{fig: avg reg} shows that
our algorithms are superior to the existing algorithms.
We can see that the performance of TV is superior than GP-UCB in the uniform setting.
In the uniform setting, the performance of TV is stable from the perspective of the variance of the cumulative regret per round.
On the other hand, the performance of GP-UCB is unstable in the uniform setting.
In the biased setting, the performance of both GP-UCB and TV are unstable.
Furthermore, it can be observed that CTV, CTV-simple, and CTV-fixed, which are our proposed methods, outperformed both TV and GP-UCB in all settings.
From Figure \ref{fig: avg reg}, we can see that CTV-fixed performs best among all methods, which is reasonable because only CTV-fixed knows all evaluation time beforehand.

It may be confusing that the performance of TV is worse than that of the proposed methods not only in the biased setting but also in the uniform setting,
The biased setting is constructed to produce the imbalances of the evaluation time.
On the other hand, the uniform setting is constructed similarly to the setting in the existing study \citep{Bogunovic_et_al_2016}.
However, there is a difference of our uniform setting and the setting of the existing study \citep{Bogunovic_et_al_2016}, which is the length of evaluation time, although it is identical for all rounds.
In our uniform setting, the all evaluation time is set to be $3$, 
but 
the existing study only considers the case where all evaluation time is set to be $1$.
Therefore, the actual performance of TV and our proposed methods are different as shown in Figure \ref{fig: avg reg}.

Except for CTV-fixed, our second proposed method CTV performs best among all methods.
The mean of the achieved cumulative regret per round of CTV is minimum in those of GP-UCB, TV and CTV-simple.
The other notable point is the standard deviation of the cumulative regret per round of CTV.
In the uniform setting, the standard deviation of CTV is smaller than that of CTV-simple.
Moreover, in the biased setting, the standard deviation of CTV is much smaller than CTV-simple with both the squared exponential kernel and the \matern (5/2) kernel.
In the biased setting by the squared exponential kernel, the standard deviation of CTV is much smaller than that of GP-UCB, and in the same setting by the \matern (5/2) kernel, that of CTV is much smaller than that of TV.
This implies that the our proposed method CTV demonstrates stable performance in all cases. 

From the left of Figure \ref{fig: step}, if we set $t=1$, TV performs best among all methods.
This is reasonable because the model of the generated data is identical to the assumed model of the objective function in TV.
It may be confusing that the performance of CTV-fixed is totally different from that of TV.
This is because that the GP posterior of CTV-fixed is constructed from the measured time data.
The measured time data contains some noises, which is critical to the performance of CTV, CTV-simple, and CTV-fixed, when the relative noise magnitude is large compared to the time data.

In Figure \ref{fig: step} (right), if we set $t = 10$, our proposed methods perform best compared to existing methods.
We note that the performances of CTV, CTV-simple, and CTV-fixed are almost the same.
This suggests that if the relative noise magnitude is small compared to the time data, it is expected that our proposed methods perform well and can outperform existing methods.

\section{Conclusion} \label{sec: conc}
We proposed a novel time-varying GP bandit algorithm 
which takes the non-constant evaluation time into account.
To the best of our knowledge, this work is the first to theoretically link the time-invariant and time-varying settings.
We also showed that a sub-linear regret can be achieved even in the time-varying setting in contrast to the proven fact in \citet{Bogunovic_et_al_2016}.
This means that the sequence of the evaluation time greatly affects the regret order.
For future work, it is important to study the effects of the sequence of the evaluation time in more detail.

\ifisICML
\else
\section*{Acknowledgements}
HI was supported by JST AIP Challenge.
NC was supported by MEXT scholarship and JST AIP Challenge.
MS was supported by KAKENHI 17H00757.
\fi

\newpage
\bibliography{ref}

\begin{thebibliography}{26}
\providecommand{\natexlab}[1]{#1}
\providecommand{\url}[1]{\texttt{#1}}
\expandafter\ifx\csname urlstyle\endcsname\relax
  \providecommand{\doi}[1]{doi: #1}\else
  \providecommand{\doi}{doi: \begingroup \urlstyle{rm}\Url}\fi

\bibitem[Auer(2002)]{Auer_2002}
P.~Auer.
\newblock Using confidence bounds for exploitation-exploration tradeoffs.
\newblock \emph{Journal of Machine Learning Research}, 3:\penalty0 397--422,
  2002.

\bibitem[Besbes et~al.(2014)Besbes, Gur, and Zeevi]{Besbes_et_al_2014}
O.~Besbes, Y.~Gur, and A.~Zeevi.
\newblock Stochastic multi-armed-bandit problem with non-stationary rewards.
\newblock \emph{In Advances in Neural Information Processing Systems}, pages
  199--207, 2014.

\bibitem[Bogunovic et~al.(2016)Bogunovic, Scarlett, and
  Cevher]{Bogunovic_et_al_2016}
I.~Bogunovic, J.~Scarlett, and V.~Cevher.
\newblock Time-varying {Gaussian} process bandit optimization.
\newblock \emph{In Proceedings of the 19th International Conference on
  Artificial Intelligence and Statistics (AISTATS)}, 2016.

\bibitem[Bogunovic et~al.(2018)Bogunovic, Scarlett, Jegelka, and
  Cevher]{Bogunovic_et_al_2018}
I.~Bogunovic, J.~Scarlett, S.~Jegelka, and V.~Cevher.
\newblock Adversarially robust optimization with {Gaussian} processes.
\newblock \emph{In Advances in Neural Information Processing Systems}, 2018.

\bibitem[Contal et~al.(2014)Contal, Perchet, and Vayatis]{Contal_et_al_2014}
E.~Contal, V.~Perchet, and N.~Vayatis.
\newblock {Gaussian} process optimization with mutual information.
\newblock \emph{In Proceedings of the 31th International Conference on Machine
  Learning}, 2014.

\bibitem[Cover and Thomas(1991)]{Cover_and_Thomas_1991}
T.~M. Cover and J.~A. Thomas.
\newblock \emph{Elements of Information Theory}.
\newblock Wiley Interscience, 1991.

\bibitem[Hariri et~al.(2015)Hariri, Mobasher, and Burke]{Hariri_et_al_2015}
Negar Hariri, Bamshad Mobasher, and Robin Burke.
\newblock Adapting to user preference changes in interactive recommendation.
\newblock In \emph{Proceedings of the 24th International Conference on
  Artificial Intelligence}, IJCAI'15, pages 4268--4274. AAAI Press, 2015.

\bibitem[Heaton and Lucas(1999)]{Heaton_and_Lucas_1999}
John Heaton and Deborah Lucas.
\newblock Stock prices and fundamentals.
\newblock \emph{NBER Macroeconomics Annual}, 14:\penalty0 213--242, 1999.

\bibitem[Henning and Schuler(2012)]{Henning_and_Schuler_2012}
P.~Henning and C.~Schuler.
\newblock Entropy search for information-efficient global optimization.
\newblock \emph{Journal of Machine Learning Research}, 13:\penalty0 1809--1837,
  2012.

\bibitem[Hern\'{a}ndez-Lobato et~al.(2014)Hern\'{a}ndez-Lobato, Hoffman, and
  Ghahramani]{Hernandez_Lobato_et_al_2014}
J.~Hern\'{a}ndez-Lobato, M.~Hoffman, and Z.~Ghahramani.
\newblock Predictive entropy search for efficient global optimization of
  black-box functions.
\newblock \emph{In Advances in Neural Information Processing Systems}, 2014.

\bibitem[Horn and Johnson(2012)]{Horn_and_Johnson_2012}
R.~A. Horn and C.~R. Johnson.
\newblock \emph{Matrix Analysis}.
\newblock Cambridge University Press, 2nd edition, 2012.

\bibitem[Jones et~al.(1998)Jones, Schonlau, and Welch]{Jones_et_al_1998}
D.~Jones, M.~Schonlau, and W.~Welch.
\newblock Expensive global optimization of expensive black-box functions.
\newblock \emph{Joumal of Global Optimization}, 13:\penalty0 455--492, 1998.

\bibitem[Klein et~al.(2017)Klein, Falkner, Bartels, Hennig, and
  Hutter]{Klein_et_al_2017}
A.~Klein, S.~Falkner, S.~Bartels, P.~Hennig, and F.~Hutter.
\newblock Fast {B}ayesian optimization of machine learning hyperparameters on
  large datasets.
\newblock \emph{In Proceedings of the 20th International Conference on
  Artificial Intelligence and Statistics (AISTATS)}, 2017.

\bibitem[Krause and Ong(2011)]{Krause_and_Ong_2011}
A.~Krause and C.~S. Ong.
\newblock Contextual {Gaussian} process bandit optimization.
\newblock \emph{In Advances in Neural Information Processing Systems}, pages
  2447--2455, 2011.

\bibitem[Liu and Nocedal(1989)]{Liu_1989}
Dong~C. Liu and Jorge Nocedal.
\newblock On the limited memory {B}{F}{G}{S} method for large scale
  optimization.
\newblock \emph{Mathematical Programming}, 45\penalty0 (1):\penalty0 503--528,
  Aug 1989.

\bibitem[Lizotte et~al.(2007)Lizotte, Wang, Bowling, and
  Schuurmans]{Lizotte_et_al_2007}
D.~J. Lizotte, T.~Wang, M.~H. Bowling, and D.~Schuurmans.
\newblock Automatic gait optimization with {Gaussian} process regression.
\newblock \emph{In International Joint Conference on Artificial Intelligence
  (IJCAI)}, pages 944--949, 2007.

\bibitem[Mockus et~al.(1978)Mockus, Tiesis, and Zilinskas]{Mockus_et_al_1978}
J.~Mockus, V.~Tiesis, and A.~Zilinskas.
\newblock The application of {B}ayesian methods for seeking the extremum.
\newblock \emph{Towards Global Optimization}, pages 117--129, 1978.

\bibitem[Rasmussen and Williams(2006)]{Rasmussen_and_Williams_2006}
C.~Rasmussen and C.~Williams.
\newblock \emph{{Gaussian} Processes for Machine Learning}.
\newblock The MIT Press, 2006.

\bibitem[Shahriari et~al.(2015)Shahriari, Swersky, Wang, Adams, and
  de~Freitas]{Shahriari_et_al_2015}
B.~Shahriari, K.~Swersky, Z.~Wang, R.~Adams, and N.~de~Freitas.
\newblock Taking the human out of the loop: A review of {B}ayesian
  optimization.
\newblock \emph{In Proceedings of the IEEE}, (1), 2015.

\bibitem[Slivkins and Upfal(2008)]{Slivkins_and_Upfal_2008}
A.~Slivkins and E.~Upfal.
\newblock Adapting to a changing environment: the brownian restless bandits.
\newblock \emph{In Conference on Learning Theory}, 2008.

\bibitem[Snoek et~al.(2012)Snoek, Larochelle, and Adams]{Snoek_et_al_2012}
J.~Snoek, H.~Larochelle, and R.~P. Adams.
\newblock Practical {B}ayesian optimization of machine learning algorithms.
\newblock \emph{In Advances in Neural Information Processing Systems}, 2012.

\bibitem[Srinivas et~al.(2010)Srinivas, Krause, Kakade, and
  Seeger]{Srinivas_et_al_2010}
N.~Srinivas, A.~Krause, S.~Kakade, and M.~Seeger.
\newblock {Gaussian} process optimization in the bandit setting: No regret and
  experimental design.
\newblock \emph{In Proceedings of the 27th International Conference on Machine
  Learning}, 2010.

\bibitem[Swersky et~al.(2013)Swersky, Snoek, and Adams]{Swersky_et_al_2013}
K.~Swersky, J.~Snoek, and R.~P. Adams.
\newblock Multi-task {Bayesian} optimization.
\newblock \emph{In Advances in Neural Information Processing Systems}, 2013.

\bibitem[Szegedy et~al.(2013)Szegedy, Zaremba, Sutskever, Bruna, Erhan,
  Goodfellow, and Fergus]{Szegedy_et_al_2013}
C.~Szegedy, W.~Zaremba, I.~Sutskever, J.~Bruna, D.~Erhan, I.~Goodfellow, and
  R.~Fergus.
\newblock Intriguing properties of neural networks.
\newblock \emph{International Conference on Learning Representations (ICLR)},
  2013.

\bibitem[Vanchinathan et~al.(2014)Vanchinathan, Nikolic, Bona, and
  Krause]{Vanchinathan_et_al_2014}
H.~P. Vanchinathan, I.~Nikolic, F.~D. Bona, and A.~Krause.
\newblock Explore-exploit in top-n recommender systems via {Gaussian}
  processes.
\newblock \emph{In Proceedings of the 8th ACM Conference on Recommender
  systems}, pages 225--232, 2014.

\bibitem[Wang and Jegelka(2017)]{Wang_and_Jegelka_2017}
Z.~Wang and S.~Jegelka.
\newblock Max-value entropy search for efficient {B}ayesian optimization.
\newblock \emph{In Proceedings of the 34th International Conference on Machine
  Learning}, 2017.

\end{thebibliography}
\ifisICML
\else 
\bibliographystyle{plainnat}
\fi

\appendix
\onecolumn

\section{Gradients of Acquisition Functions}
In this section, we show how to calculate the gradients of our proposed acquisition functions with respect to its input $x \in \cd$.
The gradients of acquisition functions can be used in the hyperparameter optimization of GPs and the optimization of acquisition functions in each round.
For details of the hyperparameter optimization and the acquisition function optimization, see existing works \citep{Rasmussen_and_Williams_2006, Snoek_et_al_2012}.

First, we give the gradient of the CTV-fixed acquisition function \eqref{formula: acq fixed}.
Assume that we have an access to the gradient of the base acquisition function
$\nabla_{x, \tau} \alpha_{\mathrm{base}}(x, \tau)$.
Using the gradient of the base acquisition function, we can calculate the gradient of the CTV-fixed acquisition function as follows.
\begin{lem}
For the CTV-fixed acquisition function, the gradient is given as follows.
\begin{align*}
    \nabla_{x} \alpha(x | \mathscr{D}_n)
    =
    \nabla_{x, \tau} \alpha_{\mathrm{base}}(x, \tau | \mathscr{D}_n) \mid_{\tau = \tau_n + t}
    =
    \nabla_{x, \tau} \alpha_{\mathrm{base}}(x, \tau_n + t | \mathscr{D}_n).
\end{align*}
\end{lem}
\begin{proof}
    The result immediately follows from the elemental calculation.
\end{proof}

Second, we give the gradient of the CTV acquisition function.
Assume that we can use the gradient of the base acquisition function
$\nabla_{x, \tau} \alpha_{\mathrm{base}}(x, \tau)$.
Using the gradient of the base acquisition function and exchanging the differentiation and integration, we can calculate the gradient of the CTV acquisition function as follows.
\begin{lem}
    For the CTV acquisition function, the gradient is given as follows.
    \begin{align*}
        \frac{\partial}{\partial x_i} \alpha(x | \mathscr{D}_n)
        =
        \int
        \left(
            \frac{\partial \alpha_{\mathrm{base}}}{\partial x_i}(x, s)
            +
            \frac{\partial \alpha_{\mathrm{base}}}{\partial \tau}(x, s)
            \frac{\partial \tau}{\partial x_i}(x, s)
        \right) \frac{e^{- s^2}}{\sqrt{\pi}} ds,
    \end{align*}
    where
    \begin{align*}
        t(x, s) &= \exp \left(
            \sqrt{2} \sigma_n^g(x) s + \mu_n^g(x)
        \right),\\
        \frac{\partial \alpha_{\mathrm{base}}}{\partial x_i}(x, s)
        &=
        \frac{\partial}{\partial x_i}
        \alpha_{\mathrm{base}} \left(x, \tau | \mathscr{D}_n \right) 
        \mid_{\tau = \tau_n + t(x, s)},\\
        \frac{\partial \alpha_{\mathrm{base}}}{\partial \tau}(x, s)
        &=
        \frac{\partial}{\partial \tau}
        \alpha_{\mathrm{base}} \left(x, \tau | \mathscr{D}_n \right) 
        \mid_{\tau = \tau_n + t(x, s)},\\
        \frac{\partial \tau}{\partial x_i}(x, s)
        &= 
        \frac{\partial}{\partial x_i} \left(
        \exp \left(
            \sqrt{2} \sigma_n^g(x) s + \mu_n^g(x)
        \right)
        \right)\\
        &=
        \exp \left(
            \sqrt{2} \sigma_n^g(x) s + \mu_n^g(x)
        \right)
        \left(
            \sqrt{2}
            \frac{\partial \sigma_n^g(x)}{x_i}
            s
            +
            \frac{\partial \mu_n^g(x)}{x_i}
        \right).
    \end{align*}
\end{lem}
\begin{proof}
    Recall that the CTV acquisition function is defined as follows.
    \begin{align*}
        \alpha(x | \mathscr{D}_n)
        =
        \int 
        \alpha_{\mathrm{base}}(x,\tau_n+t|\mathscr{D}_n) 
        p(t|x,\mathscr{D}_n)
        dt.
    \end{align*}
    Here, the posterior distribution of $t$ is a log-normal distribution as follows.
    \begin{align*}
        p(t|x,\mathscr{D}_n)
        &=
        \Lambda \left(t|\mu_n^g(x), (\sigma_n^g)^2(x) \right)\\
        &= \frac{1}{\sqrt{2\pi} \sigma_n^g(x) t}
        \exp
        \left(
        - \frac{(\log t - \mu_n^g(x))^2}{2 (\sigma_n^g(x))^2}
        \right).
    \end{align*}
    By replacement of 
    $t = t(x, s) = \exp \left(
            \sqrt{2} \sigma_n^g(x) s + \mu_n^g(x)
    \right)$, we get the following expression of the CTV acquisition function.
    \begin{align*}
        \alpha(x | \mathscr{D}_n)
        &= 
        \int 
        \alpha_{\mathrm{base}}(x,\tau_n+t|\mathscr{D}_n) 
        p(t|x,\mathscr{D}_n)
        dt\\
        &=
        \int 
        \alpha_{\mathrm{base}}(x,\tau_n+t|\mathscr{D}_n) 
        \frac{1}{\sqrt{2\pi} \sigma_n^g(x) t}
        \exp
        \left(
        - \frac{(\log t - \mu_n^g(x))^2}{2 (\sigma_n^g(x))^2}
        \right)
        dt\\
        &=
        \int 
        \alpha_{\mathrm{base}}(x,\tau_n+t(x, s)|\mathscr{D}_n) 
        \frac{
        \exp
        \left(- s^2 \right)
        t(x, s) \sqrt{2} \sigma_n^g(x)}
        {\sqrt{2\pi} \sigma_n^g(x) t(x, s)}
        ds\\
        &=
        \int 
        \alpha_{\mathrm{base}}(x,\tau_n+t(x, s)|\mathscr{D}_n) 
        \frac{e^{-s^2}}{\sqrt{\pi}}
        ds.
    \end{align*}
    If we exchange the differentiation and the integration, we get the following.
    \begin{align*}
        \frac{\partial}{\partial x_i} \alpha(x | \mathscr{D}_n)
        =
        \int \frac{\partial}{\partial x_i} 
        \alpha_{\mathrm{base}} \left(x, \tau_n + t(x,s) | \mathscr{D}_n \right)
        \frac{e^{-s^2}}{\sqrt{\pi}}
        ds.
    \end{align*}
    By the chain rule, we get the following.
    \begin{align*}
        \frac{\partial}{\partial x_i} \alpha_{\mathrm{base}} \left(x, \tau_n + t(x,s) | \mathscr{D}_n \right)
        &=
        \frac{\partial}{\partial x_i}
        \alpha_{\mathrm{base}} \left(x, \tau | \mathscr{D}_n \right) 
        \mid_{\tau = \tau_n + t(x, s)}\\
        &\hspace{10pt}+
        \frac{\partial}{\partial \tau}
        \alpha_{\mathrm{base}} \left(x, \tau | \mathscr{D}_n \right) 
        \mid_{\tau = \tau_n + t(x, s)}
        \frac{\partial}{\partial x_i}
        \left( \tau_n + t(x, s) \right)\\
        &=
        \frac{\partial \alpha_{\mathrm{base}}}{\partial x_i}(x, s)
        +
        \frac{\partial \alpha_{\mathrm{base}}}{\partial \tau}(x, s)
        \frac{\partial \tau}{\partial x_i}(x, s).
    \end{align*}
    Combining these two identities, we complete the proof.
\end{proof}

Third, we give the gradient of the CTV-simple acquisition function.
Assume that we can use the gradient of the base acquisition function
$\nabla_{x, \tau} \alpha_{\mathrm{base}}(x, \tau)$.
Using the gradient of the base acquisition function and exchanging the differentiation and integration, we can calculate the gradient of the CTV-simple acquisition function as follows.
\begin{lem}
    For the CTV-simple acquisition function, the gradient is given as follows.
    \begin{align*}
        \frac{\partial}{\partial x_i} \alpha(x | \mathscr{D}_n)
        =
            \frac{\partial \alpha_{\mathrm{base}}}{\partial x_i}(x)
            +
            \frac{\partial \alpha_{\mathrm{base}}}{\partial \tau}(x)
            \frac{\partial \tau}{\partial x_i}(x),
    \end{align*}
    where
    \begin{align*}
        t(x) &= \exp \left(
            \frac{1}{2} ((\sigma_n^g(x))^2 + \sigma^2) + \mu_n^g(x)
        \right),\\
        \frac{\partial \alpha_{\mathrm{base}}}{\partial x_i}(x)
        &=
        \frac{\partial}{\partial x_i}
        \alpha_{\mathrm{base}} \left(x, \tau | \mathscr{D}_n \right) 
        \mid_{\tau = \tau_n + t(x)},\\
        \frac{\partial \alpha_{\mathrm{base}}}{\partial \tau}(x)
        &=
        \frac{\partial}{\partial \tau}
        \alpha_{\mathrm{base}} \left(x, \tau | \mathscr{D}_n \right) 
        \mid_{\tau = \tau_n + t(x)},\\
        \frac{\partial \tau}{\partial x_i}(x)
        &= 
        \frac{\partial}{\partial x_i} \left(
        \exp \left(
            \frac{1}{2} ((\sigma_n^g(x))^2 + \sigma^2) + \mu_n^g(x)
        \right)
        \right)\\
        &=
        \exp \left(
            \frac{1}{2} ((\sigma_n^g(x))^2 + \sigma^2) + \mu_n^g(x)
        \right)
        \left(
            \sigma_n^g(x)
            \frac{\partial \sigma_n^g(x)}{x_i}
            +
            \frac{\partial \mu_n^g(x)}{x_i}
        \right).
    \end{align*}
\end{lem}
\begin{proof}
    If we use the chain rule, we get the following.
    \begin{align*}
        \frac{\partial}{\partial x_i} \alpha(x | \mathscr{D}_n)
        =
        \frac{\partial \alpha_{\mathrm{base}}}{\partial x_i}(x)
            +
            \frac{\partial \alpha_{\mathrm{base}}}{\partial \tau}(x)
            \frac{\partial \tau}{\partial x_i}(x).
    \end{align*}
    By the replacement of 
    $t = \exp \left(
            \frac{1}{2} (\sigma^2_n(x) + \sigma^2) + \mu_n(x)
    \right)$, we complete the proof.
\end{proof}

\section{Proof of Theorem \ref{thm: main}} \label{sec: proof of main thm}
We prove Theorem \ref{thm: main} in this section.

Let $\cd_n$ be the data obtained by the round $n$: $\cd_n = \{ (x_i, t_i, y_i) \}_{i=1}^n$.
Let $\tau_n = \sum_{i=1}^n t_i$.
Denote the posterior mean and variance conditioned on $\cd_n$ by $\mu_n(x, \tau)$ and $\sigma^2_n(x, \tau)$.
First, we prove the following lemma to bound the actual observation $f(x, \tau)$ by the posterior mean and posterior variance.

\begin{lem} \label{lem: obs}
    Pick $\delta \in (0, 1)$ and set $\beta_n = 2 \log \frac{\pi_n}{\delta}$
    where $\sum_{i=0}^\infty \frac{1}{\pi_n} = 1$ and $\pi_n > 0$ for all $n$.
    Then, for any occurrence of data $\cd_{n+1}$,
    with probability greater than $1 - \delta$,
    \begin{align*}
        |y_{n+1} - \mu_n(x_{n+1}, \tau_{n+1})| \le \sqrt{\beta_n} \sigma_n(x_{n+1}, \tau_{n+1})
    \end{align*}
    for all $n \ge 0$ holds.
\end{lem}
\begin{proof}
    Fix $n \ge 1$ and $x_{n+1} \in \cd, t_{n+1} \in \ct$.
    Given $\xnp, \tnp$ and conditioned on $\mathscr{D}_n$, $y_{n+1} = f(\xnp, \tnp)$ is according to the Gaussian 
    \begin{align*}
    \cn(\mu_n(\xnp, \tnp), \sigma^2_n(\xnp, \tnp)).
    \end{align*}
    
    If a random variable $X \sim \cn(0, 1)$, the following elementary probability bound holds: for any $c > 0$, 
    \begin{align*}
        \Pr (X > c)
        &= \int_{c}^\infty (2 \pi) ^ {- \frac{1}{2}} \exp \left( - \frac{x^2}{2} \right) dx\\
        &= \int_{0}^\infty (2 \pi) ^ {- \frac{1}{2}} \exp \left( - \frac{(x + c)^2}{2} \right) dx\\
        &= e^{- \frac{c^2}{2}} \int_{0}^\infty (2 \pi) ^ {- \frac{1}{2}}  \exp \left( - \frac{x^2}{2} \right) \exp(-cx) dx\\
        &\le e^{- \frac{c^2}{2}} \int_{0}^\infty (2 \pi) ^ {- \frac{1}{2}}  \exp \left( - \frac{x^2}{2} \right) dx\\
        &= e^{- \frac{c^2}{2}} \Pr (X > 0)\\
        &= \frac{e^{- \frac{c^2}{2}}}{2}.
    \end{align*}
    
    Let $X := \frac{y_{n+1} - \mu_n(x_{n+1}, \tau_{n+1})}{\sigma_n(x_{n+1}, \tau_{n+1})}$ and $c = \sqrt{\beta_n}$.
    Because $X \sim \cn(0,1)$ and $c > 0$, using the above probability bound, we obtain
    \begin{align*}
        Pr(|X| > c) 
        &\le Pr(X > c) + Pr(-X > c)\\
        &\le 2 \Pr (X > c)\\
        &\le e^{- \frac{c^2}{2}}.
    \end{align*}
    That us, for any $n \ge 0$, given $\xnp, \tnp$ and conditioned on $\mathscr{D}_n$,
    \begin{align*}
        \Pr \left( |y_{n+1} - \mu_n(x_{n+1}, \tau_{n+1})| > \sqrt{\beta_n} \sigma_n(x_{n+1}, \tau_{n+1}) \right) \le e^{- \frac{\beta_n}{2}}.
    \end{align*}
    Taking the union bound for $n \ge 0$, we obtain
    \begin{align*}
        &\Pr \left( \exists n \ge 0 \:\: |y_{n+1} - \mu_n(x_{n+1}, \tau_{n+1})| > \sqrt{\beta_n} \sigma_n(x_{n+1}, \tau_{n+1}) \right)\\
        &\le \sum_{n \ge 0} \Pr \left( |y_{n+1} - \mu_n(x_{n+1}, \tau_{n+1})| > \sqrt{\beta_n} \sigma_n(x_{n+1}, \tau_{n+1}) \right)\\
        &\le \sum_{n \ge 0} e^{- \frac{\beta_n}{2}}\\
        &= \sum_{n \ge 0} \frac{\delta}{\pi_n}\\
        &= \delta.
    \end{align*}
    This completes the proof.
\end{proof}

Next, we consider the discretization of the input domain $\cd$.
Let $\cd_n$ be some finite subset of $\cd$.
Note that $y = f(x, \tau)$ is a random variable given $x \in \cd$ and $\tau \in \ct$.
The following lemma aims to bound the observed function value $y$ by the posterior mean and posterior variance over the finite subset $\cd_n$ instead of $\cd$.

\begin{lem} \label{lem: disc}
    Pick $\delta \in (0, 1)$ and set $\beta_n = 2 \log \frac{ |\cd_n| \pi_n}{\delta}$
    where $\sum_{i=0}^\infty \frac{1}{\pi_n} = 1$ and $\pi_n > 0$ for all $n$.
    Then, for any occurrence of data $\cd_{n+1}$, for any $\tau \in \ct$,
    with probability greater than $1 - \delta$,
    \begin{align*}
        |y - \mu_n(x, \tau)| \le \sqrt{\beta_n} \sigma_n(x, \tau)
    \end{align*}
    for all $n \ge 0$ and for all $x \in \cd_n$ holds.
\end{lem}
\begin{proof}
    The proof of this lemma is almost identical that is Lemma \ref{lem: obs}.
    We use the union bound over not only $n \ge 0$ but also $ x \in \cd_n$.
\end{proof}

Next, we give the concrete example of the discretization of $\cd$ and consider actual discretization bound for actual observations.

\begin{lem} \label{lem: actual disc}
    Let $\cd = [0, r]^d$.
    Pick $\delta \in (0, 1)$ and 
    set 
    \begin{align*}
    \beta_n = 2 \log \frac{ 2 \pi_n}{\delta} + 2 d \log \left( dn^2br \sqrt{\log \frac{2 \pi_n ad}{\delta}} \right)
    \end{align*}
    where $\sum_{i=0}^\infty \frac{1}{\pi_n} = 1$ and $\pi_n > 0$ for all $n$,
    and for any sample path $f$ from $\cg \cp (0, k)$, there exists some $a > 0$ and $b > 0$ such that for any $j \in [d]$, for any $\tau \in \ct$ and for any $L > 0$,
    \begin{align*}
        \Pr \left( \sup_{x \in \cd} \left| \frac{\partial f(x, \tau)}{\partial x^{(j)}}\right| > L\right) \le a \exp \left( - \frac{L^2}{b^2} \right).
    \end{align*}
    Let $x_\tau^* = \argmax_{x \in \cd} f(x, \tau)$
    and 
    $[x]_n$ denote the closest point in $\cd_n$ to $x$ with respect to the $l_1$-norm.
    Let $y_\tau^* = f(x_\tau^*, \tau)$.
    Then, for any $\tau \in \ct$,
    with probability greater than $1 - \delta$,
    \begin{align*}
        |y_\tau^* - \mu_n ([x_\tau^*]_{n+1}, \tau)| \le \sqrt{\beta_n} \sigma_n([x_\tau^*]_{n+1}, \tau) + \frac{1}{n^2}
    \end{align*}
    holds for any $n \ge 0$.
\end{lem}
\begin{proof}
    First, by assumption and the union bound, for any $\tau \in \ct$, 
    \begin{align*}
        \Pr \left( 
            \forall j \in [d], \: 
            \forall x \in \cd, \:
            \left| \frac{\partial f(x, \tau)}{\partial x^{(j)}}\right| < L \right)
            \ge 1 - a d \exp \left( - \frac{L^2}{b^2} \right),
    \end{align*}
    which implies that for any $\tau \in \ct$, with probability greater than $1 - a d \exp \left( - \frac{L^2}{b^2} \right)$, we have that
    \begin{align}
        \label{formula: in main thm, f bound} \forall x, x' \in \cd, \: \left| f(x, \tau) - f(x', \tau) \right| \le L \| x - x' \|_1.
    \end{align}

    Here we specify the discretization $\cd_n$ of $\cd$.
    Let $\xi_n  = d n^2 b r \sqrt{\log \frac{2 \pi_n ad}{\delta}}$ and
    $R_n$ be the subset of $\rr^d$ define as follows:
    \begin{align*}
        R_n = \{ x \in \rr^d \mid 
        \forall j \in [d], \:
        \exists k \in \zz, \: 
        x_j = \frac{r}{\xi_n} k
        \}.
    \end{align*}
    The set $R_n$ is the set of uniformly divided points of $\rr^d$ with interval $\frac{r}{\xi_n}$.
    We define the discretization $\cd_n$ by
    \begin{align*}
        \cd_n = \cd \cap R_n.
    \end{align*}
    By construction, the discretization $\cd_n$ is the finite subset of $\cd$ of size at most $(\xi_n)^d$.
    The following inequality also holds for $\cd_n$ by construction:
    for any $x \in \cd_n$, let $[x]_n$ be the closet point in $\cd_n$ to x with respect to the $l_1$-norm, then
    \begin{align}
        \label{formula: in main thm, x bound} \| x - [x]_n \|_1 \le \frac{r d}{\xi_n}.
    \end{align}
    
    Combining (\ref{formula: in main thm, f bound}) and (\ref{formula: in main thm, x bound}), we get the following probability bound.
    Setting $L = b \sqrt{\log \frac{2 \pi_n ad}{\delta}}$ in \eqref{formula: in main thm, f bound}, we have for any $\tau \in \ct$ with probability greater than $1 - \delta / 2 \pi _n$,
    \begin{align*}
        \forall x, x' \in \cd, \: \left| f(x, \tau) - f(x', \tau) \right| \le b \sqrt{\log \frac{2 \pi_n ad}{\delta}} \| x - x' \|_1.
    \end{align*}
    Hence, for any $\tau \in \ct$ with probability greater than $1 - \delta / 2 \pi _n$,
    \begin{align}
        \label{formula: in main thm, f bound 1/n^2} \forall x \in \cd, \: \left| f(x, \tau) - f([x]_n, \tau) \right| \le d b r \sqrt{\log \frac{2 \pi_n ad}{\delta}} \frac{1}{\xi_n} = \frac{1}{n^2}.
    \end{align}
    
    Because the size of $\cd_n$ is $|\cd_n| \le (\xi_n)^d = \left( d n^2 b r \sqrt{\log \frac{2 \pi_n ad}{\delta}} \right)^d$, 
    \begin{align*}
        \beta_n 
        &= 2 \log \frac{ 2 \pi_n}{\delta} + 2 d \log \left( dn^2br \sqrt{\log \frac{2 \pi_n ad}{\delta}} \right)\\
        &\ge 2 \log \frac{ 2 \pi_n}{\delta} + 2 \log |\cd_n|\\
        &= 2 \log \frac{|\cd_n| \pi_n}{\delta/2}.
    \end{align*}
    \allowdisplaybreaks
    We obtain 
    \begin{align}
        \nonumber &\Pr \left( \exists n \ge 0, \: |y_\tau^* - \mu_n ([x_\tau^*]_{n+1}, \tau)| > \sqrt{\beta_n} \sigma_n([x_\tau^*]_{n+1}, \tau) + \frac{1}{n^2} \right)\\
        \nonumber &\le \Pr \left( \exists n \ge 0, 
        \: |y_\tau^* - f([x_\tau^*]_n, \tau)| + |f([x_\tau^*]_n, \tau) - \mu_n ([x_\tau^*]_{n+1}, \tau)|  
        \right. \\
        \label{formula: in main thm, in actual disc, ineq 1} &\hspace{200pt}\left.
        > \sqrt{\beta_n} \sigma_n([x_\tau^*]_{n+1}, \tau) + \frac{1}{n^2} \right)\\ 
        \nonumber &= 
        1 - \Pr \left(
            \forall n \ge 0,
            |y_\tau^* - f([x_\tau^*]_n, \tau)| 
            + 
            |f([x_\tau^*]_n, \tau) - \mu_n ([x_\tau^*]_{n+1}, \tau)|
            \right. \\ 
            \nonumber &\hspace{200pt}\left.\: 
            \le
            \sqrt{\beta_n} \sigma_n([x_\tau^*]_{n+1}, \tau) + \frac{1}{n^2}
        \right)\\
        \nonumber &\le
        1 - \Pr \left(\right. \\ 
            \nonumber &\hspace{30pt}\left.\left\{
                \forall n \ge 0, \:
                |y_\tau^* - f([x_\tau^*]_n, \tau)|
                \le 
                \frac{1}{n^2}
            \right\}
            \cap \right. \\ 
            \nonumber &\hspace{30pt}\left.\left\{
                \forall n \ge 0, \:
                |f([x_\tau^*]_n, \tau) - \mu_n ([x_\tau^*]_{n+1}, \tau)|
                \le 
                \sqrt{\beta_n} \sigma_n([x_\tau^*]_{n+1}, \tau)
            \right\}
        \right)\\
        \nonumber &=
        \Pr \left(
            \left\{
                \exists n \ge 0, \:
                |y_\tau^* - f([x_\tau^*]_n, \tau)|
                > 
                \frac{1}{n^2}
            \right\}
            \cup \right. \\
            \nonumber &\hspace{30pt} \left. \left\{
                \exists n \ge 0, \:
                |f([x_\tau^*]_n, \tau) - \mu_n ([x_\tau^*]_{n+1}, \tau)|
                > 
                \sqrt{\beta_n} \sigma_n([x_\tau^*]_{n+1}, \tau)
            \right\}
        \right)\\
        \nonumber &\le \Pr \left( \exists n \ge 0, \: |y_\tau^* - f([x_\tau^*]_n, \tau)| > \frac{1}{n^2} \right) + 
        \\
        \label{formula: in main thm, in actual disc, ineq 2} &\hspace{30pt}
        \Pr \left( \exists n \ge 0, \: |f([x_\tau^*]_n, \tau) - \mu_n ([x_\tau^*]_{n+1}, \tau)| > \sqrt{\beta_n} \sigma_n([x_\tau^*]_{n+1}, \tau)\right)\\
        \nonumber &\le \sum_{n \ge 0} \Pr \left(|f(x_\tau^*, \tau) - f([x_\tau^*]_n, \tau)| > \frac{1}{n^2} \right) +
        \\
        \label{formula: in main thm, in actual disc, ineq 3} &\hspace{30pt} \Pr \left( \exists n \ge 0, \: |f([x_\tau^*]_n, \tau) - \mu_n ([x_\tau^*]_{n+1}, \tau)| > \sqrt{\beta_n} \sigma_n([x_\tau^*]_{n+1}, \tau)\right)\\
        \nonumber &\le \sum_{n \ge 0} \Pr \left( \exists x \in \cd,  \:|f(x, \tau) - f([x]_n, \tau)| > \frac{1}{n^2} \right) + \\
        \nonumber &\hspace{30pt} 
        \Pr \left( \exists n \ge 0, \: |f([x_\tau^*]_n, \tau) - \mu_n ([x_\tau^*]_{n+1}, \tau)| > \sqrt{\beta_n} \sigma_n([x_\tau^*]_{n+1}, \tau)\right)\\
        \label{formula: in main thm, in actual disc, ineq 4} &\le \sum_{n \ge 0} \frac{\delta}{2 \pi_n} + \frac{\delta}{2}\\
        \nonumber &\le \delta,
    \end{align}
    where 
    the inequality \eqref{formula: in main thm, in actual disc, ineq 1} follows from the triangle inequality,
    the inequality \eqref{formula: in main thm, in actual disc, ineq 2} and
    the inequality \eqref{formula: in main thm, in actual disc, ineq 3} follows from the union bound, and
    the inequality \eqref{formula: in main thm, in actual disc, ineq 4} follows from the equation \eqref{formula: in main thm, f bound 1/n^2} and the result of Lemma \ref{lem: disc}.
    This completes the proof.
\end{proof}

Next, we derive a high-probability upper bound for the simple regret $r_n$.

\begin{lem} \label{lem: simple bound}
    Pick $\delta \in (0, 1)$ and set 
    $\beta_n = 2 \log \frac{4 \pi_n}{\delta} + 2d \log \left( d n^2 b r \sqrt{\log \frac{4 \pi_n a d }{\delta}} \right)$,
    where $\sum_{i=0}^\infty \frac{1}{\pi_n} = 1$ and $\pi_n > 0$ for all $n$.
    Then, with probability greater than $1 - \delta$,
    \begin{align*}
        r_n \le 2\sqrt{\beta_{n-1}} \sigma_{n-1}(x_n, \tau_n) + \frac{1}{n^2}
    \end{align*}
    holds for any $n \ge 0$.
\end{lem}
\begin{proof}
    We $\delta := \delta / 2$ in Lemma \ref{lem: obs} and \ref{lem: actual disc}, so that the following two events holds for any $\tau \in \ct$ with probability greater than $1 - \delta$:
    \begin{align*}
        \forall n \ge 0, \: |y_{n+1} - \mu_n(x_{n+1}, \tau_{n+1})| \le \sqrt{\beta_n} \sigma_n(x_{n+1}, \tau_{n+1}),
    \end{align*}
    and
    \begin{align*}
        \forall n \ge 0, \: |y_\tau^* - \mu_n ([x_\tau^*]_{n+1}, \tau)| \le \sqrt{\beta_n} \sigma_n([x_\tau^*]_{n+1}, \tau) + \frac{1}{n^2}.
    \end{align*}
    By using the definition of the UCB acquisition function and these lemmas, we obtain
    \begin{align*}
        r_n
        &= f(x_{\tau_n}^*, \tau_n) - f(x_n, \tau_n)\\
        &\le \mu_{n-1}([x_{\tau_n}^*]_n, \tau_n) + \sqrt{\beta_{n-1}} \sigma_{n-1} ([x_{\tau_n}^*]_n, \tau_n) + \frac{1}{n^2} - f(x_n, \tau_n)\\
        &\le \mu_{n-1}(x_n, \tau_n) + \sqrt{\beta_{n-1}} \sigma_{n-1} (x_n, \tau_n) + \frac{1}{n^2} - f(x_n, \tau_n)\\
        &\le 2 \sqrt{\beta_{n-1}} \sigma_{n-1} (x_n, \tau_n) + \frac{1}{n^2}.
    \end{align*}
    This comletes the proof.
\end{proof}

Next, we bound the cumulative regret $R_n$ using the maximum information gain $\tilde{\gamma}_n$.

\begin{lem} \label{lem: cum bound using max info gain}
    Pick $\delta \in (0,1)$ and set
    $\beta_n = 2 \log \frac{4 \pi_n}{\delta} + 2d \log \left( d n^2 b r \sqrt{\log \frac{4 \pi_n a d }{\delta}} \right)$.
    Let $C = \frac{8}{\log (1 + \sigma^{-2})}$
    and $\tilde{\gamma}_n$ denotes the maximum information gain defined as \eqref{formula: maximum information gain}.
    Then, with probability greater than $1 - \delta$,
    \begin{align*}
        R_n \le \sqrt{C \beta_n \tau_n \tilde{\gamma}_n} + 2
    \end{align*}
    holds for any $n \ge 0$.
\end{lem}
\begin{proof}
\allowdisplaybreaks
    \begin{align}
        R_n
        \nonumber &= \sum_{i=1}^n r_i\\
        \label{formula: in main thm, cum, eq 1} &\le \sum_{i=1}^n \left( 2 \sqrt{\beta_{i-1}} \sigma_i(x_i, \tau_i) + \frac{1}{i^2}\right)\\
        \label{formula: in main thm, cum, eq 2} &\le \sum_{i=1}^n \left(4 \beta_{i-1} \sigma_{i-1}^2(x_i, \tau_i) \right)^{\frac{1}{2}} + 2\\
        \label{formula: in main thm, cum, eq 3} &\le \left( \sum_{i=1}^n 4 \beta_{i-1} \sigma_{i-1}^2(x_i, \tau_i) \right)^{\frac{1}{2}} \left( \sum_{i=1}^n 1 \right)^{\frac{1}{2}}  + 2\\
        \nonumber &= \left( \sum_{i=1}^n 4 \beta_{i-1} \sigma_{i-1}^2(x_i, \tau_i) \right)^{\frac{1}{2}} \sqrt{n}  + 2\\
        \nonumber &= \left( \sum_{i=1}^n 4 \beta_{i-1} \sigma^{2} \sigma^{-2} \sigma_{i-1}^2(x_i, \tau_i) \right)^{\frac{1}{2}} \sqrt{n}  + 2\\
        \label{formula: in main thm, cum, eq 4} &\le \left( \sum_{i=1}^n 4 \beta_{i-1} \sigma^{2} \frac{\sigma^{-2}}{\log ( 1 + \sigma^{-2})} \log \left( 1 + \sigma^{-2}\sigma_{i-1}^2(x_i, \tau_i) \right)
        \right)^{\frac{1}{2}} \sqrt{n}  + 2\\
        \nonumber &\le \left(  \frac{1}{2} C  \beta_n \sum_{i=1}^n \log \left( 1 + \sigma^{-2}\sigma_{i-1}^2(x_i, \tau_i) \right)
        \right)^{\frac{1}{2}} \sqrt{n}  + 2\\
        \label{formula: in main thm, cum, eq 5} &= \sqrt{C \beta_n n \tilde{\gamma}_n} + 2,
    \end{align}
    where
    the inequality \eqref{formula: in main thm, cum, eq 1} follows from Lemma \ref{lem: simple bound},
    the inequality \eqref{formula: in main thm, cum, eq 2} follows from the inequality $\sum_{i=1}^n\frac{1}{i^2} = \frac{\pi^2}{6} < 2$, 
    the inequality \eqref{formula: in main thm, cum, eq 3} follows from the Cauchy-Schwarz's inequality, 
    the inequality \eqref{formula: in main thm, cum, eq 4} follows from tha fact that for any $s^2 \in [0, \sigma^{-2}]$, $s^2 \le \frac{\sigma^{-2}}{\log (1 + \sigma^{-2})} \log (1 + s^2)$, and 
    the equation \eqref{formula: in main thm, cum, eq 5} follows from $\tilde{\gamma}_n = \max_{x_1, \ldots, x_n} \frac{1}{2} \sum_{i=1}^n \log \left( 1 + \sigma^{-2}\sigma_{i-1}^2(x_i, \tau_i) \right)$ by Lemma 5.3 in \citet{Srinivas_et_al_2010}.
    This completes the proof.
\end{proof}

Finally, we evaluate the maximum information gain by using the maximum space information gain and the evaluation time uniformity.

\begin{lem} \label{lem: out-of-time cum bound}
    Assume that the space kernel $k_{\mathrm{space}}$ satisfies for any $x, x' \in \cd$,
    \begin{align}
        \label{formula: in main thm, gamma bound, k_space} k_{\mathrm{space}}(x, x') \le 1 
    \end{align}
    and the time kernel $k_{\mathrm{time}}$ satisfies for some $\epsilon \in [0,1]$, for any $\tau, \tau' \in \ct$, 
    \begin{align*}
    1 - k_{\mathrm{time}}(\tau, \tau') \le \epsilon | \tau - \tau'|.
    \end{align*}
    Let $\phi(x) = \min \left( x, \log x + \frac{1}{x} \right)$ for any $x \in \rr$.
    Let $\{d_i \}_{i=0}^N \subset \{ 0, \ldots, n \}$ be the partition of $\{ 0, \ldots, n \}$ as follows.
    \begin{align*}
        0 = d_0 < d_1 < d_2 < \cdots < d_{N-1} < d_N = n.
    \end{align*}
    Set $M_i = d_i - d_{i-1}$ and $M = \max_{i=1, \ldots, N} M_i$.
    Let
    \begin{align*}
        \ct_{i} = \{ \tau_{d_{i-1}+1}, \ldots, \tau_{d_i} \} \subset \{ \tau_k \}_{k=1}^n,
    \end{align*}
    for any $i \in [n]$.
    Then, 
    \begin{align*}
        \tilde{\gamma}_n \le 
        N \gamma_M 
        +
        \frac{1}{2} \sum_{i = 1}^{N} M_i \phi \left( \sigma^{-2} \epsilon \sqrt{\frac{C_{\epsilon, \ct_{i}}}{M_i}}
        \right).
    \end{align*}
\end{lem}
\begin{proof}
    For any positive semi-definite square matrix $A$, we denote $\lambda_i(A)$ as the $i$-th eigenvalue of $A$ in descending order.
    Pick any $i \in [n]$.
    Recall that 
    $K_n = (k((x_k, \tau_k), (x_l, \tau_l)))_{k, l = 1}^n$ and 
    $\tilde{K}_n = (k_{\mathrm{space}}(x_k, x_l))_{k, l = 1}^n$.
    Define the index set $\ci_i = \{ d_{i-1} + 1, \ldots, d_i \}$.
    For submatrices 
    $K^i = (k((x_k, \tau_k), (x_l, \tau_l)))_{k, l \in \ci_i} \in \mathbb{R}^{M_i \times M_i}$ and
    $\tilde{K}^i = (k_{\mathrm{space}}(x_k, x_l))_{k, l \in \ci_i} \in \mathbb{R}^{M_i \times M_i}$,
    let $A^i$ be 
    \begin{align*}
        A^i = \tilde{K}^i - K^i = K^i \circ (J^i - 1_i),
    \end{align*}
    where $\circ$ denotes the hadamard product,
    $J^i = (k_{\mathrm{time}}(\tau_k, \tau_l))_{k,l \in \ci_i} \in \mathbb{R}^{M_i \times M_i}$,
    and $1_i = (1)_{k,l=1}^{M_i} \in \mathbb{R}^{M_i \times M_i}$.
    Let $\Delta_k^{i} = \lambda_k (\tilde{K}^i) - \lambda_k (K^i)$.

    First, we bound the squared sum of $\{ \Delta_k^{i} \}_{k=1}^{M_i}$ for any $i \in [N]$ as follows.
    \begin{align}
        \nonumber \sum_{k=1}^{M_i} \left( \Delta_k^{i} \right)^2
        &= \sum_{k=1}^{M_i} \left( \lambda_k (\tilde{K}^i) - \lambda_k (K^i) \right)^2\\
        \nonumber &= \left\| diag (\lambda_1 (\tilde{K}^i), \ldots, \lambda_i (\tilde{K}^i)) - diag (\lambda_1 (K^i), \ldots, \lambda_i (K^i)) \right\|_F^2\\
        \label{formula: in main thm, gamma bound, eq1} &\le \left\| \tilde{K}^i - K^i \right\|_F^2\\
        \nonumber &= \left\| K^i \circ (J^i - 1_i) \right\|_F^2\\
        \label{formula: in main thm, gamma bound, eq2} &\le \left\| J^i - 1_i \right\|_F^2\\
        \nonumber &\le \sum_{k \in \ci_i} \sum_{l \in \ci_i} \left( 1 - (1 - \epsilon) ^ {\frac{|\tau_k - \tau_l|}{2}}\right) ^2\\
        \nonumber &\le \sum_{k \in \ci_i} \sum_{l \in \ci_i} \min \left( 1, \epsilon^2 (\tau_k - \tau_l)^2 \right)\\
        \nonumber &= \sum_{\tau_k \in \ct_{i}} \sum_{\tau_l \in \ct_{i}} \min \left( 1, \epsilon^2 (\tau_k - \tau_l)^2 \right)\\
        \label{formula: final delta bound} &= \epsilon^2 C_{\epsilon, \ct_{i}},
    \end{align}
    where
    the inequality \eqref{formula: in main thm, gamma bound, eq1} follows from Lemma \ref{lem: mirsky} in the following with $U = \tilde{K}^i, V = K^i$, and $\| \cdot \| = \| \cdot \|_F$, and
    the inequality \eqref{formula: in main thm, gamma bound, eq2} follows from the fact that each entry of $K^i$ is smaller than $1$ by \eqref{formula: in main thm, gamma bound, k_space}.
    
    \begin{lem}[Mirsky's theorem \citep{Horn_and_Johnson_2012}, Cor. 7.4.9.3] \label{lem: mirsky}
        Choose any $i \times i$ matrices $U$ and $V$.
        For any matrix $U$, we denote $\lambda_i(U)$ as the $i$-th eigenvalue of $U$ in descending order.
        Let $\| \cdot \|$ be any unitaly invariant norm.
        Then, we have
        \begin{align*}
            \| \mathrm{diag}(\lambda_1 (U), \ldots, \lambda_i (U)) - \mathrm{diag}(\lambda_1 (V), \ldots, \lambda_i (V)) \| \le \| U - V \|.
        \end{align*}
    \end{lem}
   
    Next, we bound the maximum information gain $\tilde{\gamma}_n$ by the maximum space information gain $\gamma_M$ and the evaluation time uniformity $C_{\epsilon, \ct_i}$.
    Recall that $\bff_n = (f(x_j, \tau_j))_{j=1}^n$ and $\by_n = (y_j)_{j=1}^n$ with $y_j = f(x_j, \tau_j) + z_j$.
    We define the block vectors of 
    $\bff_n$ and $\by_n$ for any 
    $i \in [N]$ 
    as follows.
    \begin{align*}
        \bff^i &= \left( f(x_{d_{i-1} + 1}, \tau_{d_{i-1} + 1}), \ldots, f(x_{d_i}, \tau_{d_i}) \right)\\
        \by^i &= \left( y_{d_{i-1} + 1}, \ldots, y_{d_i} \right).
    \end{align*}
    If we use the chain rule for conditional mutual information and the independence of the noise sequence $\{ z_n\}$, we have
    \begin{align*}
    \tilde{I}(\by_n; \bff_n) = \sum_{i=1}^{N} \tilde{I}(\by^i; \bff^i).
    \end{align*}
    Note that all $\by^i$ and $\bff^i$ have length at most $M$.
    We maximize each $\tilde{I}(\by^i; \bff^i)$ ($i \in \left\{ 1, \ldots, N \right\}$) with respect to $x_{d_{i-1} + 1}, \ldots, x_{d_i}$
    If we maximize both sides over $x_1, \ldots, x_n$, we obtain
    \begin{align*}
        \tilde{\gamma}_n \le \sum_{i=1}^{N} \tilde{\gamma}^i,
    \end{align*}
    where $\tilde{\gamma}^i$ is the maximum information gain of the $i$-th block defined by
    \begin{align*}
        \tilde{\gamma}^i = \max_{x_{d_{i-1}+1}, \ldots, x_{d_i}} \tilde{I}(\by^i ; \bff^i).
    \end{align*}
    We also define the maximum space information gain of the $i$-th block $\gamma^i$ by
    \begin{align*}
        \gamma^i = \frac{1}{2} \max_{x_{d_{i-1}+1}, \ldots, x_{d_i}} \sum_{k=1}^{M_i} \log (1 + \sigma^{-2}\lambda_k(K^i))
    \end{align*}
    Note that each $\tilde{\gamma}^i$ depends on $\ct_{i}$, but each $\gamma^i$ does not depend on $\ct_{i}$, because the matrix $K^i$ only depends on space information.
    Therefore, $\gamma^i$ only depends on the kernel $k_{\mathrm{space}}$ and the size of the matrix $M_i$.
    We denote $\gamma^i$ as $\gamma_{M_i}$, which is upper bounded by $\gamma_M$.
    \begin{align}
        \gamma_{M_i} \le \gamma_M. \label{formula: in main thm, gamma M bound}
    \end{align}
    
    Using the evaluation (\ref{formula: final delta bound}), we can bound the maximum information gain of the $i$-th block as follows.
    \allowdisplaybreaks
    \begin{align}
        \tilde{\gamma}^i
        \nonumber &= \max_{x_{d_{i-1}+1}, \ldots, x_{d_i}} \tilde{I}(\by^i; \bff^i)\\
        \nonumber &= \frac{1}{2} \max_{x_{d_{i-1}+1}, \ldots, x_{d_i}} \log \det ( I + \sigma^{-2} \tilde{K}^i )\\
        \nonumber &= \frac{1}{2} \max_{x_{d_{i-1}+1}, \ldots, x_{d_i}} \log \prod_{k=1}^{M_i} (1 + \sigma^{-2}\lambda_k(\tilde{K}^i))\\
        \nonumber &= \frac{1}{2} \max_{x_{d_{i-1}+1}, \ldots, x_{d_i}} \sum_{k=1}^{M_i} \log (1 + \sigma^{-2}\lambda_k(\tilde{K}^i))\\
        \label{formula: in main thm, gamma bound, gamma0} &= \frac{1}{2} \max_{x_{d_{i-1}+1}, \ldots, x_{d_i}} \sum_{k=1}^{M_i} \log (1 + \sigma^{-2}\lambda_k(K^i + A^i))\\
        \label{formula: in main thm, gamma bound, gamma1} &\le \frac{1}{2} \max_{x_{d_{i-1}+1}, \ldots, x_{d_i}} \sum_{k=1}^{M_i} \log (1 + \sigma^{-2}\lambda_k(K^i) + \sigma^{-2}\Delta_k^{i})\\
        \label{formula: in main thm, gamma bound, gamma2} &\le \frac{1}{2} \max_{x_{d_{i-1}+1}, \ldots, x_{d_i}} \sum_{k=1}^{M_i} \log (1 + \sigma^{-2}\lambda_k(K^i)) + \frac{1}{2} \sum_{k=1}^{M_i} \log (1 + \sigma^{-2}\Delta_k^{i})\\
        \label{formula: in main thm, gamma bound, gamma3} &= \gamma_{M_i} + \frac{1}{2} \sum_{k=1}^{M_i} \log (1 + \sigma^{-2}\Delta_k^{i})\\
        \label{formula: in main thm, gamma bound, gamma4} &\le \gamma_{M_i} + \frac{M_i}{2} \log \left( 1 + \sigma^{-2} \frac{1}{M_i} \sum_{k=1}^{M_i} \Delta_k^{i} \right)\\
        \label{formula: in main thm, gamma bound, gamma5} &\le \gamma_{M_i} + \frac{M_i}{2} \log \left( 1 + \sigma^{-2} \sqrt{\frac{1}{M_i} \sum_{k=1}^{M_i} \left( \Delta_k^{i} \right)^2 }\right)\\
        \label{formula: in main thm, gamma bound, gamma6} &\le \gamma_{M_i} + \frac{M_i}{2} \log \left( 1 + \sigma^{-2} \sqrt{\frac{\epsilon^2 C_{\epsilon, \ct_{i}}}{M_i}}\right)\\
        \nonumber &= \gamma_{M_i} + \frac{M_i}{2} \log \left( 1 + \sigma^{-2} \epsilon \sqrt{\frac{C_{\epsilon, \ct_{i}}}{M_i}}\right)\\
        \label{formula: in main thm, gamma bound, gamma7} &\le \gamma_{M} + \frac{M_i}{2} \log \left( 1 + \sigma^{-2} \epsilon \sqrt{\frac{C_{\epsilon, \ct_{i}}}{M_i}}\right),
    \end{align}
    where 
    the equality \eqref{formula: in main thm, gamma bound, gamma0} follows from the definition of $A_i^j$, 
    the inequality \eqref{formula: in main thm, gamma bound, gamma1} follows from the definition of $\Delta_k^{i}$,
    the inequality \eqref{formula: in main thm, gamma bound, gamma2} follows from the fact that for any $a > 0, b > 0$, $\log(1 + a + b) \le \log(1 + a) + \log(1 + b)$,
    the equality \eqref{formula: in main thm, gamma bound, gamma3} follows from the definition of $\gamma^i = \gamma_{M_i}$,
    the inequality \eqref{formula: in main thm, gamma bound, gamma4} follows from the Jensen's inequality for $\log(1 + x)$, 
    the inequality \eqref{formula: in main thm, gamma bound, gamma5} follows from the Jensen's inequality for $x^2$,
    the inequality \eqref{formula: in main thm, gamma bound, gamma6} follows from \eqref{formula: final delta bound}, and
    the inequality \eqref{formula: in main thm, gamma bound, gamma7} follows from \eqref{formula: in main thm, gamma M bound}.

    If $\sigma^{-2} \epsilon \sqrt{\frac{C_{\epsilon, \ct_{i}}}{i}} \le 1$, using $\log(1 + x) \le x$, we get
    \begin{align*}
        \gamma_M + \frac{M_i}{2} \log \left( 1 + \sigma^{-2} \epsilon \sqrt{\frac{C_{\epsilon, \ct_{i}}}{M_i}}\right)
        \le
        \gamma_M + \frac{M_i}{2} \left(\sigma^{-2} \epsilon \sqrt{\frac{C_{\epsilon, \ct_{i}}}{M_i}} \right).
    \end{align*}
    On the other hand, if $\sigma^{-2} \epsilon \sqrt{\frac{C_{\epsilon, \ct_{i}}}{M_i}} \ge 1$, using $\log (1 + x) = \log x + \log \left(1 + \frac{1}{x} \right) \le \log x + \frac{1}{x}$, we get
    \begin{align*}
        &\gamma_M + \frac{M_i}{2} \log \left( 1 + \sigma^{-2} \epsilon \sqrt{\frac{C_{\epsilon, \ct_{i}}}{M_i}}\right) \\
        &\le
        \gamma_M + \frac{M_i}{2} \left( \log \left( \sigma^{-2} \epsilon \sqrt{\frac{C_{\epsilon, \ct_{i}}}{M_i}} \right) + \left( \sigma^{-2} \epsilon \sqrt{\frac{C_{\epsilon, \ct_{i}}}{M_i}} \right)^{-1}\right).
    \end{align*}
    We combine both cases using 
    $\phi(x) = \min \left( x, \log x + \frac{1}{x} \right)$ as follows.
    \begin{align*}
        \gamma_M + \frac{M_i}{2} \log \left( 1 + \sigma^{-2} \epsilon \sqrt{\frac{C_{\epsilon, \ct_{i}}}{M_i}}\right)
        \le
         \gamma_M + \frac{M_i}{2} \phi \left(\sigma^{-2} \epsilon \sqrt{\frac{C_{\epsilon, \ct_{i}}}{M_i}} \right).
    \end{align*}
    This completes the proof.
\end{proof}

Using Lemma \ref{lem: cum bound using max info gain} and \ref{lem: out-of-time cum bound}, we obtain the result of the Theorem \ref{thm: main}.

\section{Proof of Lemmas \ref{lem: uni uni} and \ref{lem: bia uni}} \label{sec: proof of lem}
\begin{proof}[Proof of Lemma \ref{lem: uni uni}]
    Recall that $\ct'$ is a subset of $\{ \tau_k \}_{k=1}^n$ with $i$ consecutive elements, that is, for some $k_0 \le n - i$,
    \begin{align*}
        \ct' = \{ \tau_{k_0 + 1}, \ldots, \tau_{k_0 + i} \}.
    \end{align*}
    Since $t_k = \frac{T}{n}$ for any $k \in [n]$ in the uniform setting, 
    $\tau_k = \frac{T}{n}k$.
    Therefore, 
    \begin{align}
        C_{\epsilon, \ct'} 
        \nonumber &= \sum_{k=1}^i \sum_{l=1}^i \min \left(\frac{1}{\epsilon^2}, \left( \tau_{k_0 + k} - \tau_{k_0 + l} \right)^2 \right)\\
        \nonumber &= \sum_{k=1}^i \sum_{l=1}^i \min \left(\frac{1}{\epsilon^2}, \left( \frac{T}{n}k - \frac{T}{n}l\right)^2 \right)\\
        &= \frac{T^2}{n^2} \sum_{k=1}^i \sum_{l=1}^i \min \left(\frac{n^2}{\epsilon^2 T^2}, (k-l)^2\right) \label{formula: sum uni uni}.
    \end{align}
    To calculate the above sum \eqref{formula: sum uni uni}, we consider two cases: (1) $\frac{n}{\epsilon T} \ge i$ and (2) $\frac{n}{\epsilon T} \le i$.
    If $\frac{n}{\epsilon T} \ge i$, then all minimums inside of \eqref{formula: sum uni uni} are equal to $(k-l)^2$.
    Therefore, 
    \allowdisplaybreaks
    \begin{align*}
        C_{\epsilon, \ct'} 
        &= \frac{T^2}{n^2} \sum_{k=1}^i \sum_{l=1}^i (k-l)^2\\
        &= \frac{T^2}{n^2} \frac{1}{6} i^2 (i^2 - 1)\\
        &= 
        \frac{1}{6} \frac{T^2}{n^2} i^2 (i^2 - 1).
    \end{align*}
    On the other hand, if $\frac{n}{\epsilon T} \le i$, then the minimums inside of \eqref{formula: sum uni uni} are equal to 
    $(k-l)^2$ when $|k-l| \le \frac{n}{\epsilon T}$ 
    and equal to 
    $\frac{n^2}{\epsilon^2 T^2}$ when $|k-l| > \frac{n}{\epsilon T}$.
    For simplicity, we assume that $\frac{n}{\epsilon T}$ is integer.
    Therefore, 
    \allowdisplaybreaks
    \begin{align*}
        C_{\epsilon, \ct'} 
        &= \frac{T^2}{n^2} \sum_{k=1}^i \sum_{l=1}^i \min \left( \frac{n^2}{\epsilon^2 T^2}, (k-l)^2 \right)\\
        &= \frac{T^2}{n^2} \left( 
            \sum_{|k-l| \le \frac{n}{\epsilon T} } (k-l)^2
            + 
            \sum_{|k-l| > \frac{n}{\epsilon T}} \frac{n^2}{\epsilon^2 T^2}
            \right)\\
        &= \frac{T^2}{n^2} \left( 
            \sum_{l=1}^\frac{n}{\epsilon T} l^2 (i-l)
            + 
            \frac{n^2}{\epsilon^2 T^2} \sum_{l=1}^{i-\frac{n}{\epsilon T} - 1} l
            \right)\\
        &= \frac{T^2}{n^2} \frac{n}{\epsilon T} \left( 
            i^2 \frac{n}{\epsilon T} 
            - \frac{4}{3}i \left(\frac{n}{\epsilon T}\right)^2
            + \frac{1}{2} \left(\frac{n}{\epsilon T}\right)^3
            + \frac{i}{3}
            - \frac{1}{2} \frac{n}{\epsilon T}
            \right)\\
        &= 
        \frac{T}{\epsilon n}
        \left(
            \frac{1}{2} \left( \frac{n}{T \epsilon} \right)^3 
            - \frac{4}{3} i \left( \frac{n}{T \epsilon} \right)^2
            + \left( i^2 - \frac{1}{2} \right)\frac{n}{T \epsilon}
            + \frac{i}{3} \right).
    \end{align*}
    This completes the proof.
\end{proof}
\begin{proof}[Proof of Lemma \ref{lem: bia uni}]
    Recall that $\ct'$ be a subset of $\{ \tau_k \}_{k=1}^n$ with $i$ consecutive elements, that is, for some $k_0 \le n - i$,
    \begin{align*}
        \ct' = \{ \tau_{k_0 + 1}, \ldots, \tau_{k_0 + i} \}.
    \end{align*}
    Since $t_{n_0} = T$ and $t_i = 0$ for $i \neq n_0$ in the extremely biased setting, $\tau_i = 0$ for $i < n_0$ and $\tau_i = T$ for $i \ge n_0$.
    Therefore, if $\tau_{n_0} \not \in \ct'$, 
    \begin{align*}
        C_{\epsilon, \ct'}
        = \sum_{k=1}^i \sum_{l=1}^i 0 = 0.
    \end{align*}
    On the other hand, if $\tau_{n_0} \in \ct'$,
    \allowdisplaybreaks
    \begin{align*}
        C_{\epsilon, \ct'}
        &= \sum_{k=1}^i \sum_{l=1}^i \min \left( \frac{1}{\epsilon^2}, (\tau_{k_0 + k} - \tau_{k_0 + l})^2 \right)\\
        &= \sum_{k=1}^i \left( 
        \sum_{l=1}^{n_0 - k_0 -1} \min \left( \frac{1}{\epsilon^2}, \tau_{k_0 + k}^2 \right)
        +
        \sum_{l=n_0 - k_0}^i \min \left( \frac{1}{\epsilon^2}, (\tau_{l_0 + k} - T)^2 \right)
        \right)\\
        &= \sum_{k=1}^i \left( 
        (n_0-k_0-1) \min \left( \frac{1}{\epsilon^2}, \tau_{k_0 + k}^2 \right)
        + \right. \\
        &\hspace{100pt} \left.  (k_0 + i - n_0 + 1) \min \left( \frac{1}{\epsilon^2}, (\tau_{k_0 + k} - T)^2 \right)
        \right) \\
        &= 
        (n_0-k_0-1) (k_0 + i-n_0+1)\min \left( \frac{1}{\epsilon^2}, T^2 \right)
        +\\
        &\hspace{100pt} (k_0 + i-n_0+1) (n_0-k_0-1) \min \left( \frac{1}{\epsilon^2}, T^2 \right)\\
        &= 2 (n_0-k_0-1) (k_0+i-n_0+1)\min \left( \frac{1}{\epsilon^2}, T^2 \right).
    \end{align*}
    This completes the proof.
\end{proof}

\section{Proof of Theorem \ref{thm: derived}} \label{sec: proof of derived thm}
We prove Theorem \ref{thm: derived} in this section.
There are four cases in Theorem \ref{thm: derived}, so we prove them in order.
Recall that $T$ is $T = \tau_n = \sum_{i=1}^n t_i$, and the notation $\tilde{O}(\cdot)$ denotes the asymptotic growth rate up to logarithmic factors and suppose that the time kernel is the following special case of the exponential kernel:
    \begin{align*}
        k_{\mathrm{time}}(\tau, \tau') = (1 - \epsilon) ^ {\frac{|\tau - \tau'|}{2}}.
    \end{align*} 

\begin{lem} \label{lem: sq uniform}
    Suppose that the space kernel is the squared exponential kernel and evaluation time $\{ t_i \}_{i=1}^n$ are uniform, that is, $t_i = \frac{T}{n}$.
    Then,
    if $\epsilon T < n^{- \frac{3}{2}}$, then
    \begin{align*}
        R_n = \tilde{O} (\sqrt{n}),
    \end{align*}
    and if $n^{- \frac{3}{2}} \le \epsilon T \le n$, then
    \begin{align*}
        R_n = \tilde{O} 
        \left( 
            n^\frac{4}{5} 
            T^\frac{1}{5} 
            \epsilon^\frac{1}{5} 
        \right),
    \end{align*}
    and if $n < \epsilon T$, then
    \begin{align*}
        R_n = \tilde{O} \left(n 
        \left(1
            +
            \left( 
                \frac{\epsilon T}{n}
            \right)^{\frac{1}{2}}
            \right)
        \right)
    \end{align*}
    with high probability.
\end{lem}
\begin{proof}
Pick any $i \in [n]$.
We consider the following partition of $\{ 0, \ldots, n \}$ in the uniform setting.
\begin{align*}
    \{ d_j \}_{j=1}^{\lceil n / i \rceil} = \{ \min (ij, n) \}_{j=1}^{\lceil n / i \rceil}.
\end{align*}
This corresponds to the following finite subset of $\ct$:
\begin{align*}
        \ct_{i,j} = \{ \tau_{i(j-1)+1}, \ldots, \tau_{\min(ij, n)} \} \subset \{ \tau_k \}_{k=1}^n,
\end{align*}
for any $i \in [n]$ and $j \in [\lceil n / i\rceil]$.
Here, we can assume that $\min (ij, n) = ij$, since if $\min (ij, n) = n$ we add $ij - n$ dummy elements so that $\ct_{i,j}$ has $i$ consecutive elements.
Such dummy elements are
$
\frac{T}{n}(n+k)
$ for $k \in \{ 1, \ldots, ij-n\}$.

    We set $\ct' = \ct_{i, j}$ in Lemma \ref{lem: uni uni}. 
    In the uniform setting, we get the evaluation time uniformity $C_{\epsilon, \ct_{i,j}}$ is if $\frac{n}{\epsilon T} \ge i$
    \begin{align*}
        C_{\epsilon, \ct_{i,j}} = \frac{1}{6} \frac{T^2}{n^2} i^2 (i^2 - 1), \numberthis \label{formula: in sq uniform, c bound1}
    \end{align*}
    and if $\frac{n}{\epsilon T} \le i$
    \begin{align*}
        C_{\epsilon, \ct_{i,j}} = 
        \frac{T}{\epsilon n}
        \left(
            \frac{1}{2} \left( \frac{n}{T \epsilon} \right)^3 
            - \frac{4}{3} i \left( \frac{n}{T \epsilon} \right)^2
            + \left( i^2 - \frac{1}{2} \right)\frac{n}{T \epsilon}
            + \frac{i}{3} \right). \numberthis \label{formula: in sq uniform, c bound2}
    \end{align*}
    We denote the RHS of \eqref{formula: in sq uniform, c bound1} as $C^1_{\epsilon, i}$ and that of \eqref{formula: in sq uniform, c bound2} as $C^2_{\epsilon, i}$.
    
    In this setting, the upper bound of the cumulative regret \eqref{formula: reg upper bound} becomes as follows:
    \begin{align*}
        R_n &\le 
        \sqrt{C \beta_n n 
       \left( \left\lceil \frac{n}{i} \right\rceil  \gamma_i 
       +
       \frac{i}{2} \sum_{j=1}^{\lceil n / i \rceil} \phi \left( \sigma^{-2} \epsilon \sqrt{\frac{C_{\epsilon, \ct_{i,j}}}{i}} \right) \right) } + 2 \numberthis \label{formula: in sq uniform, reg bound}.
    \end{align*}
    \\
    First, we consider two cases,
    $\epsilon T < n^{-\frac{3}{2}}$ ({\bf Case 1}) and
    $n^{-\frac{3}{2}} \le \epsilon T < 1$ ({\bf Case 2}).
    In these cases, $\frac{n}{\epsilon T} \frac{1}{i} > \frac{n}{i} > 1$ for any $i \in [n]$.
    This means $\frac{n}{\epsilon T} > i$.
    Therefore, all $C_{\epsilon, \ct_{i,j}}$ in the sum of \eqref{formula: in sq uniform, reg bound} are equal to $C^1_{\epsilon, i}$.
    By substituting \eqref{formula: in sq uniform, c bound1} into \eqref{formula: in sq uniform, reg bound}, we get the following.
    \begin{align*}
        R_n &\le
        \sqrt{
            C \beta_n n \left(
                \left\lceil \frac{n}{i} \right\rceil  \gamma_i 
       +
       \frac{i}{2} \sum_{j=1}^{\lceil n / i \rceil} \phi \left( \sigma^{-2} \epsilon \sqrt{\frac{C_{\epsilon, \ct_{i,j}}}{i}} \right)
            \right)
        } +2 \\
        &= 
        \sqrt{
            C \beta_n n \left(
                \left\lceil \frac{n}{i} \right\rceil  \gamma_i 
       +
       \frac{i}{2} \sum_{j=1}^{\lceil n / i \rceil} \phi \left( \sigma^{-2} \epsilon \sqrt{\frac{C^1_{\epsilon, i}}{i}} \right)
            \right)
        } + 2\\
        &= 
        \sqrt{
            C \beta_n n \left(
                \left\lceil \frac{n}{i} \right\rceil  \gamma_i 
       +
       \frac{i}{2} \left\lceil \frac{n}{i} \right\rceil \phi \left( \sigma^{-2} \epsilon \sqrt{\frac{C^1_{\epsilon, i}}{i}} \right)
            \right)
        } + 2\\
        &= 
        \sqrt{
            C \beta_n n \left(
                \left\lceil \frac{n}{i} \right\rceil  \gamma_i 
       +
       \frac{i}{2} \left\lceil \frac{n}{i} \right\rceil \phi \left( \sigma^{-2} \epsilon \sqrt{\frac{\frac{1}{6} \frac{T^2}{n^2} i^2 (i^2 - 1)}{i}} \right)
            \right)
        } + 2. \numberthis \label{formula: in sq uniform, case1 bound1}
    \end{align*}
    For the squared exponential kernel, the maximum space information gain $\gamma_i$ is $\tilde{O}(1)$ \citep{Srinivas_et_al_2010}.
    By substituting this result into \eqref{formula: in sq uniform, case1 bound1} and simplifying it, we get the following.
    \begin{align*}
        R_n
        &\le
        \sqrt{
            C \beta_n n \left(
                \left\lceil \frac{n}{i} \right\rceil  \gamma_i 
       +
       \frac{i}{2} \left\lceil \frac{n}{i} \right\rceil \phi \left( \sigma^{-2} \epsilon \sqrt{\frac{\frac{1}{6} \frac{T^2}{n^2} i^2 (i^2 - 1)}{i}} \right)
            \right)
        } + 2\\
        &\le
        C_n
        n
        \sqrt{
            \frac{1}{i}
            +
            i
            \frac{1}{i}
            \phi\left(
                \sqrt{
                    \epsilon^2 \frac{T^2}{n^2} i^3
                }
            \right)
        }\\
        &\le
        C_n
        n
        \sqrt{
            \frac{1}{i}
            +
            \phi\left(
                \frac{\epsilon T}{n}
                i^{\frac{3}{2}}
            \right)
        }.
    \end{align*}
    Here, we introduce the constant $C_n$ satisfying $C_n = \tilde{O}(1)$. As a result, we get the following.
    \begin{align*}
        R_n 
        &= \tilde{O}
        \left(
            n \sqrt{
                \frac{1}{i}
                +
                \phi \left(
                    \frac{\epsilon T}{n} i^{\frac{3}{2}} 
                \right)
            }
        \right). \numberthis \label{formula: in sq uniform, case1 bound2}
    \end{align*}
    We minimize the RHS with respect to $i \in [n]$.
    Recall that 
    \begin{align*}
    \phi(x) = \min \left( x, \log x + \frac{1}{x} \right).
    \end{align*}
    This function takes the value of $x$ when $x \le 1$, 
    but takes the value of $\log x + \frac{1}{x}$ when $1 \le x$.
    We get the the following.
    \begin{align*}
    &\min_{
        i \in [n]
    } 
    \left(\frac{1}{i}
                +
                \phi \left(
                    \frac{\epsilon T}{n} i^{\frac{3}{2}} 
                \right)
    \right)\\
    &= \min 
    \left(
        \min_{\frac{\epsilon T}{n} i^{\frac{3}{2}} \le 1} 
        \left( \frac{1}{i}
            +
            \frac{\epsilon T}{n} i^{\frac{3}{2}} \right),
        \min_{1 \le \frac{\epsilon T}{n} i^{\frac{3}{2}}} 
        \left( \frac{1}{i}
            +
            \log \left(\frac{\epsilon T}{n} i^{\frac{3}{2}} \right)
            +
            \left( \frac{\epsilon T}{n} i^{\frac{3}{2}} \right)^{-1} \right)
    \right)
    \end{align*}
    First, we consider 
    $
    \min_{\frac{\epsilon T}{n} i^{\frac{3}{2}} \le 1} 
        \left( \frac{1}{i}
            +
            \frac{\epsilon T}{n} i^{\frac{3}{2}} \right)
    $.
    Since $1 > \epsilon T$ in both {\bf Case 1} and {\bf Case 2}, $\left( \frac{n}{\epsilon T} \right)^{\frac{2}{5}} \ge 1$ holds.
    In {\bf Case 1}, the following holds.
    \begin{align*}
    &\epsilon T < n^{- \frac{3}{2}}\\
    &\Leftrightarrow \frac{1}{\epsilon T} > n^\frac{3}{2}\\
    &\Leftrightarrow \frac{n}{\epsilon T} > n^\frac{5}{2}\\
    &\Leftrightarrow \left(
        \frac{n}{\epsilon T}
    \right)^\frac{2}{5} > n.
    \end{align*}
    Therefore, in {\bf Case 1}, we bound the minimum of
    $
    \frac{1}{i}
    +
    \frac{\epsilon T}{n} i^\frac{3}{2}
    $
    by its value when
    $i = n$ as follows.
    \begin{align*}
        \min_{\frac{\epsilon T}{n} i^{\frac{3}{2}} \le 1} 
        \left( \frac{1}{i}
            +
            \frac{\epsilon T}{n} i^{\frac{3}{2}} 
        \right)
        &=
        \tilde{O} \left(
            \frac{1}{n}
            +
            \frac{\epsilon T}{n} n^\frac{3}{2}
        \right)\\
        &=
        \tilde{O} \left(
            \frac{1}{n}
            +
            \frac{n^{-\frac{3}{2}}}{n} n^\frac{3}{2}
        \right)\\
        &=
        \tilde{O} \left( \frac{1}{n} \right).
    \end{align*}
    On the other hand, in {\bf Case 2}, the following holds.
    \begin{align*}
        &n^{- \frac{3}{2}}\le \epsilon T < 1\\
        &\Leftrightarrow 1 <
        \left(
            \frac{n}{\epsilon T}
        \right)^\frac{2}{5} \le n.
    \end{align*}
    In addition,
    since $
    \left(
        \frac{n}{\epsilon T}
    \right)^\frac{2}{5}
    $ is less than
    $
    \left(
        \frac{n}{\epsilon T}
    \right)^\frac{2}{3}
    $,
    $
    i = 
    \left(
        \frac{n}{\epsilon T}
    \right)^\frac{2}{5}
    $ satisfies
    $\frac{\epsilon T}{n} i^{\frac{3}{2}} \le 1$.
    Therefore, in {\bf Case 2}, we bound the minimum of
    $
    \frac{1}{i}
    +
    \frac{\epsilon T}{n} i^\frac{3}{2}
    $
    by its value when
    $
    i =
    \left(
        \frac{n}{\epsilon T}
    \right)^\frac{2}{5}
    $
    as follows.
    \begin{align*}
        \min_{\frac{\epsilon T}{n} i^{\frac{3}{2}} \le 1}
    \left(
        \frac{1}{i}
        +
        \frac{\epsilon T}{n}
        i^\frac{3}{2}
    \right)
    =
    \tilde{O}\left(
        \left(
            \frac{\epsilon T}{n}
        \right)^\frac{2}{5}
    \right).
    \end{align*}
    As a result, the minimum of
    $
    \frac{1}{i}
    +
    \frac{\epsilon T}{n} i^\frac{3}{2}
    $
    over
    $
    \frac{\epsilon T}{n} i^{\frac{3}{2}} \le 1
    $
    is bounded as follows.
    In {\bf Case 1}, 
    \begin{align*}
        \min_{\frac{\epsilon T}{n} i^{\frac{3}{2}} \le 1}
    \left(
        \frac{1}{i}
        +
        \frac{\epsilon T}{n}
        i^\frac{3}{2}
    \right)
    &=
    \tilde{O}
    \left(
        \frac{1}{n}
    \right).
    \end{align*}
    In {\bf Case 2},
    \begin{align*}
        \min_{\frac{\epsilon T}{n} i^{\frac{3}{2}} \le 1}
    \left(
        \frac{1}{i}
        +
        \frac{\epsilon T}{n}
        i^\frac{3}{2}
    \right)
    &=
    \tilde{O}
    \left(
        \left(
            \frac{\epsilon T}{n}
        \right)
        ^\frac{2}{5}
    \right).
    \end{align*}
    Second, we consider
    $
    \min_{
        1 
        \le 
        \frac{\epsilon T}{n} i^{\frac{3}{2}}
    } 
    \left( 
        \frac{1}{i}
        +
        \log \left(
            \frac{\epsilon T}{n} 
            i^{\frac{3}{2}} 
        \right)
        +
        \left( 
            \frac{\epsilon T}{n} 
            i^{\frac{3}{2}} 
        \right)^{-1} 
    \right).
    $
    In both cases, we bound the minimum by the value when $i = n$ as follows.
    \begin{align*}
        &\min_{
        1 
        \le 
        \frac{\epsilon T}{n} i^{\frac{3}{2}}
        }
        \left(
            \frac{1}{i}
            +
            \log\left(
                \frac{\epsilon T}{n} i^\frac{3}{2}
            \right)
            +
            \left(
                \frac{\epsilon T}{n} i^\frac{3}{2}
            \right)^{-1}
        \right)\\
        &=
        \tilde{O}
        \left(
            \frac{1}{n}
            +
            \log \left(
                \frac{\epsilon T}{n} n^\frac{3}{2}
            \right)
            +
            \left(
                \frac{\epsilon T}{n} n^\frac{3}{2}
            \right)^{-1}
        \right)\\
        &=
        \tilde{O}
        \left(
            \frac{1}{n}
            +
            1
            +
            \frac{1}{\epsilon T \sqrt{n}}
        \right)\\
        &=
        \tilde{O} \left(
            1 + \frac{1}{\epsilon T}
        \right).
    \end{align*}
    At the end of these cases {\bf Case 1} and {\bf Case 2}, the minimum of \eqref{formula: in sq uniform, case1 bound2} over $i \in [n]$ is bounded as follows.
    In {\bf Case 1},
    \begin{align*}
        R_n &= \tilde{O} 
        \left( 
        n 
        \sqrt{
            \min 
            \left(
                \frac{1}{n},
                1 + \frac{1}{\epsilon T} 
            \right) 
        }
        \right)\\
        &= 
        \tilde{O} 
        \left( 
            n 
            \sqrt{
            \frac{1}{n}
            }
        \right)\\
        &= 
        \tilde{O} 
        \left( 
            \sqrt{n}
        \right). \numberthis \label{formula: in case 1, reg}
    \end{align*}
    This is the end of {\bf Case 1}.
    On the other hand, in {\bf Case 2},
    \begin{align*}
        R_n &= \tilde{O} 
        \left( 
        n 
        \sqrt{
            \min 
            \left(
                \left(
                    \frac{\epsilon T}{n}
                \right)^\frac{2}{5},
                1 + \frac{1}{\epsilon T} 
            \right) 
        }
        \right)\\
        &= 
        \tilde{O} 
        \left( 
            n
            \sqrt{
                \left(
                    \frac{\epsilon T}{n}
                \right)^\frac{2}{5}
            }
        \right)\\
        &= 
        \tilde{O} 
        \left( 
            n^\frac{4}{5} T^\frac{1}{5} \epsilon^\frac{1}{5} 
        \right). \numberthis \label{formula: in case 2, reg}
    \end{align*}
    This is the end of {\bf Case 2}.
    \\
    {\bf Case 3 and Case 4:} 
    Consider the case of $1 \le \epsilon T \le n$ ({\bf Case 3}) 
    and $n < \epsilon T$ ({\bf Case 4}).
    We minimize the upper bound of the cumulative regret \eqref{formula: in sq uniform, reg bound} with respect to $i \in [n]$.
    In {\bf Case 3}, we divide the range of $i$ into two sections: 
    $1 \le i \le \frac{n}{\epsilon T}$
    and
    $\frac{n}{\epsilon T} \le i \le n$.
    We consider the minimum of \eqref{formula: in sq uniform, reg bound} in these two ranges, and then combine those results.
    On the other hand, in {\bf Case 4}, since $\frac{n}{\epsilon T} < 1$, we cannot consider the range $1 \le i \le \frac{n}{\epsilon T}$.
    We minimize \eqref{formula: in sq uniform, reg bound} directly.
    \\
    {\bf Case 3:}
    In this case, 
    since $1 \le \epsilon T \le n$,
    there is a constant $i_0 \in [n]$ which satisfies $\frac{n}{\epsilon T} < i_0 \le \frac{n}{\epsilon T} + 1$,
    that is,
    we define $i_0$ by
    \begin{align}
        i_0 = \left\lfloor \frac{n}{\epsilon T} \right\rfloor + 1. \label{formula: in sq uniform, case2 def i_0}
    \end{align}
    We will discuss later the case where there does not exist such a constant $i_0$
    in {\bf Case 4}.
    We divide the range of $i$ into two sections:
    $1 \le i \le i_0$
    and
    $i_0 \le i \le n$.
    \\
    {\bf Case 3-1:} If $1 \le i \le i_0$, $C_{\epsilon, \ct_{i,j}}$ in the sum of \eqref{formula: in sq uniform, reg bound} are equal to $C^1_{\epsilon, i}$.
    By substituting \eqref{formula: in sq uniform, c bound1} into \eqref{formula: in sq uniform, reg bound}, we get the following.
    \begin{align*}
        R_n &\le
        \sqrt{
            C \beta_n n \left(
                \left\lceil \frac{n}{i} \right\rceil  \gamma_i 
       +
       \frac{i}{2} \sum_{j=1}^{\lceil n / i \rceil} \phi \left( \sigma^{-2} \epsilon \sqrt{\frac{C_{\epsilon, \ct_{i,j}}}{i}} \right)
            \right)
        } +2 \\
        &= 
        \sqrt{
            C \beta_n n \left(
                \left\lceil \frac{n}{i} \right\rceil  \gamma_i 
       +
       \frac{i}{2} \sum_{j=1}^{\lceil n / i \rceil} \phi \left( \sigma^{-2} \epsilon \sqrt{\frac{C^1_{\epsilon, i}}{i}} \right)
            \right)
        } + 2\\
        &= 
        \sqrt{
            C \beta_n n \left(
                \left\lceil \frac{n}{i} \right\rceil  \gamma_i 
       +
       \frac{i}{2} \left\lceil \frac{n}{i} \right\rceil \phi \left( \sigma^{-2} \epsilon \sqrt{\frac{C^1_{\epsilon, i}}{i}} \right)
            \right)
        } + 2\\
        &= 
        \sqrt{
            C \beta_n n \left(
                \left\lceil \frac{n}{i} \right\rceil  \gamma_i 
       +
       \frac{i}{2} \left\lceil \frac{n}{i} \right\rceil \phi \left( \sigma^{-2} \epsilon \sqrt{\frac{\frac{1}{6} \frac{T^2}{n^2} i^2 (i^2 - 1)}{i}} \right)
            \right)
        } + 2. \numberthis \label{formula: in sq uniform, case1 bound3-1}
    \end{align*}
    For the squared exponential kernel, the maximum space information gain $\gamma_i$ is $\tilde{O}(1)$ \citep{Srinivas_et_al_2010}.
    By substituting this result into \eqref{formula: in sq uniform, case1 bound3-1} and simplifying it, we get the following.
    \begin{align*}
        R_n
        &\le
        \sqrt{
            C \beta_n n \left(
                \left\lceil \frac{n}{i} \right\rceil  \gamma_i 
       +
       \frac{i}{2} \left\lceil \frac{n}{i} \right\rceil \phi \left( \sigma^{-2} \epsilon \sqrt{\frac{\frac{1}{6} \frac{T^2}{n^2} i^2 (i^2 - 1)}{i}} \right)
            \right)
        } + 2\\
        &\le
        C_n
        n
        \sqrt{
            \frac{1}{i}
            +
            i
            \frac{1}{i}
            \phi\left(
                \sqrt{
                    \epsilon^2 \frac{T^2}{n^2} i^3
                }
            \right)
        }\\
        &\le
        C_n
        n
        \sqrt{
            \frac{1}{i}
            +
            \phi\left(
                \frac{\epsilon T}{n}
                i^{\frac{3}{2}}
            \right)
        }.
    \end{align*}
    Here, we introduce the constant $C_n$ satisfying $C_n = \tilde{O}(1)$. As a result, we get the following.
    \begin{align*}
        R_n 
        &= \tilde{O}
        \left(
            n \sqrt{
                \frac{1}{i}
                +
                \phi \left(
                    \frac{\epsilon T}{n} i^{\frac{3}{2}} 
                \right)
            }
        \right). \numberthis \label{formula: in case 3-1, reg}
    \end{align*}
    We minimize the RHS with respect to $i \le i_0$.
    Recall that 
    \begin{align*}
    \phi(x) = \min \left( x, \log x + \frac{1}{x} \right).
    \end{align*}
    This function takes the value of $x$ when $x \le 1$, 
    but takes the value of $\log x + \frac{1}{x}$ when $1 \le x$.
    We get the the following.
    \begin{align*}
    &\min_{
        1 \le i \le i_0
    } 
    \left(\frac{1}{i}
                +
                \phi \left(
                    \frac{\epsilon T}{n} i^{\frac{3}{2}} 
                \right)
    \right)\\
    &= \min 
    \left(
        \min_{\frac{\epsilon T}{n} i^{\frac{3}{2}} \le 1} 
        \left( \frac{1}{i}
            +
            \frac{\epsilon T}{n} i^{\frac{3}{2}} \right),
        \min_{1 \le \frac{\epsilon T}{n} i^{\frac{3}{2}}} 
        \left( \frac{1}{i}
            +
            \log \left(\frac{\epsilon T}{n} i^{\frac{3}{2}} \right)
            +
            \left( \frac{\epsilon T}{n} i^{\frac{3}{2}} \right)^{-1} \right)
    \right)
    \end{align*}
    In this case, since
    $1 \le \frac{n}{\epsilon T}$,
    $1 \le
    \left(
        \frac{n}{\epsilon T}
    \right)^\frac{2}{5}
    \le
    \left(
        \frac{n}{\epsilon T}
    \right)^\frac{2}{3}
    \le
    \frac{n}{\epsilon T}
    $.
    The$i = 
    \left(
        \frac{n}{\epsilon T}
    \right)^\frac{2}{5}
    $ satisfies
    $\frac{\epsilon T}{n} i^{\frac{3}{2}} \le 1$.
    Therefore, in {\bf Case 3-1}, we bound the minimum of
    $
    \frac{1}{i}
    +
    \frac{\epsilon T}{n} i^\frac{3}{2}
    $
    by its value when
    $
    i =
    \left(
        \frac{n}{\epsilon T}
    \right)^\frac{2}{5}
    $
    as follows.
    \begin{align*}
        \min_{\frac{\epsilon T}{n} i^{\frac{3}{2}} \le 1}
    \left(
        \frac{1}{i}
        +
        \frac{\epsilon T}{n}
        i^\frac{3}{2}
    \right)
    =
    \tilde{O}\left(
        \left(
            \frac{\epsilon T}{n}
        \right)^\frac{2}{5}
    \right). \numberthis \label{formula: in case 3-1, left bound}
    \end{align*}
    We combine \eqref{formula: in case 3-1, reg} and \eqref{formula: in case 3-1, left bound} and get the following in {\bf Case 3-1}.
    \begin{align*}
        R_n
        &=
        \tilde{O}
        \left(
            n \sqrt{
            \left(
                \frac{\epsilon T}{n}
            \right)^\frac{2}{5}
            }
        \right)\\
        &=
        \tilde{O}
        \left(
            n^\frac{4}{5}
            T^\frac{1}{5}
            \epsilon^\frac{1}{5}
        \right).
    \end{align*}
    \\
    {\bf Case 3-2}
    On the other hand, if $i_0 \le i \le n$, all $C_{\epsilon, \ct_{i, j}}$ in the sum of \eqref{formula: in sq uniform, reg bound} are equal to $C^2_{\epsilon, i}$.
    By substituting \eqref{formula: in sq uniform, c bound2} in to \eqref{formula: in sq uniform, reg bound}, we get the following.
    \begin{align*}
        R_n &\le
        \sqrt{
            C \beta_n n \left(
                \left\lceil \frac{n}{i} \right\rceil  \gamma_i 
       +
       \frac{i}{2} \sum_{j=1}^{\lceil n / i \rceil} \phi \left( \sigma^{-2} \epsilon \sqrt{\frac{C_{\epsilon, \ct_{i,j}}}{i}} \right)
            \right)
        } +2 \\
        &= 
        \sqrt{
            C \beta_n n \left(
                \left\lceil \frac{n}{i} \right\rceil  \gamma_i 
       +
       \frac{i}{2} \sum_{j=1}^{\lceil n / i \rceil} \phi \left( \sigma^{-2} \epsilon \sqrt{\frac{C^2_{\epsilon, i}}{i}} \right)
            \right)
        } + 2\\
        &= 
        \sqrt{
            C \beta_n n \left(
                \left\lceil \frac{n}{i} \right\rceil  \gamma_i 
       +
       \frac{i}{2} \left\lceil \frac{n}{i} \right\rceil \phi \left( \sigma^{-2} \epsilon \sqrt{\frac{C^2_{\epsilon, i}}{i}} \right)
            \right)
        } + 2\\
        &= 
        \sqrt{
            C \beta_n n \left(
                \left\lceil \frac{n}{i} \right\rceil  \gamma_i 
       +
       \frac{i}{2} A
            \right)
        } + 2, \numberthis \label{formula: in sq uniform, case2-2 bound1}
    \end{align*}
    where $A$ is defined as follows.
    \begin{align*}
        A = 
        \left\lceil \frac{n}{i} \right\rceil \phi \left( \sigma^{-2} \epsilon \sqrt{
       \frac{
       \frac{T}{\epsilon n}
                        \left(
                            \frac{1}{2} \left( \frac{n}{T \epsilon} \right)^3 
                            - \frac{4}{3} i \left( \frac{n}{T \epsilon} \right)^2
                            + \left( i^2 - \frac{1}{2} \right)\frac{n}{T \epsilon}
                            + \frac{i}{3} 
                        \right)
       }
       {i}} \right)
    \end{align*}
    For the squared exponential kernel, the maximum space information gain $\gamma_i$ is $\tilde{O}(1)$ \citep{Srinivas_et_al_2010}.
    By substituting this result into \eqref{formula: in sq uniform, case2-2 bound1}, we get the following.
    \begin{align*}
        R_n
        &=
        \tilde{O}
        \left(
            n \sqrt{
                \frac{1}{i}
                +
                \phi \left(
                    \sqrt{
                        i
                        +
                        \frac{1}{i}
                        \left(
                            \frac{n}{\epsilon T}
                        \right)^2
                        +
                        \frac{\epsilon T}{n}
                    }
                \right)
            }
        \right)\\
        &:= \tilde{O} 
        \left(
            n \sqrt{\psi(i)}
        \right). \numberthis \label{formula: in case3-2, def of psi}
    \end{align*}
    To minimize the RHS in \eqref{formula: in case3-2, def of psi} with respect to $i_0 \le i \le n$, we minimize the function $\psi(i)$ with respect to $i_0 \le i \le n$.
    $\psi(i)$ contains part
    $
    \phi
    \left(
                    \sqrt{
                        i
                        +
                        \frac{1}{i}
                        \left(
                            \frac{n}{\epsilon T}
                        \right)^2
                        +
                        \frac{\epsilon T}{n}
                    }
                \right)$.
    This part can be expanded by the definition of the function $\phi(x)$:
    \begin{align*}
        \phi(x) = \min
        \left(
        x,
        \log x + \frac{1}{x}
        \right),
    \end{align*}
    and
    the fact that
    \begin{align*}
        i + \frac{1}{i} \left( \frac{n}{\epsilon T} \right)^2
        &\ge 2 \sqrt{i \frac{1}{i} \left( \frac{n}{\epsilon T} \right)^2} \numberthis \label{formula: in sq uniform, case2 lem1}\\
        &= \frac{2n}{\epsilon T}\\
        &\ge i_0 \numberthis \label{formula: in sq uniform, case2 lem2}\\
        &\ge 1.
    \end{align*}
    We get the following.
    \begin{align*}
        &\phi \left(
                    \sqrt{
                        i
                        +
                        \frac{1}{i}
                        \left(
                            \frac{n}{\epsilon T}
                        \right)^2
                        +
                        \frac{\epsilon T}{n}
                    }
                \right)\\
        &=
        \log
                \left(
                    \sqrt{
                        i
                        +
                        \frac{1}{i}
                        \left(
                            \frac{n}{\epsilon T}
                        \right)^2
                        +
                        \frac{\epsilon T}{n}
                    }
                \right)  
                +
                \left(
                    \sqrt{
                        i
                        +
                        \frac{1}{i}
                        \left(
                            \frac{n}{\epsilon T}
                        \right)^2
                        +
                        \frac{\epsilon T}{n}
                    }
                \right)^{-1}.
    \end{align*}
    Then, $\psi(i)$ becomes the following.
    \begin{align*}
        \psi(i)
        &=
        \frac{1}{i}
        +
        \log
                \left(
                    \sqrt{
                        i
                        +
                        \frac{1}{i}
                        \left(
                            \frac{n}{\epsilon T}
                        \right)^2
                        +
                        \frac{\epsilon T}{n}
                    }
                \right)  
                +
                \left(
                    \sqrt{
                        i
                        +
                        \frac{1}{i}
                        \left(
                            \frac{n}{\epsilon T}
                        \right)^2
                        +
                        \frac{\epsilon T}{n}
                    }
                \right)^{-1}.
    \end{align*}
    The minimum of $\psi(i)$
    over $i_0 \le i \le n $
    satisfies that 
    $
    \min_{i_0 \le i \le n} \psi(i)
    \le
    \psi \left( i_0 \right)
    $.
    $\psi \left( i_0 \right)$ is
    \begin{align*}
        &\psi \left( i_0 \right)\\
        &=
        \psi \left(
                \left\lfloor \frac{n}{\epsilon T} \right\rfloor + 1
        \right)\\
        &=
        \tilde{O}
        \left(
            \frac{\epsilon T}{n}
            +
            \log \sqrt{
                \frac{n}{\epsilon T}
                +
                \frac{\epsilon T}{n}
                \frac{n^2}{\epsilon^2 T^2}
                +
                \frac{\epsilon T}{n}
            }
            +
            \frac{1}{
            \sqrt{
                \frac{n}{\epsilon T}
                +
                \frac{\epsilon T}{n}
                \frac{n^2}{\epsilon^2 T^2}
                +
                \frac{\epsilon T}{n}
            }
            }
        \right)\\
        &=
        \tilde{O}
        \left(
            \frac{
            \epsilon T}
            {n}
            +
            1
            +
            \frac{
            \left(
                \epsilon T
            \right)^\frac{1}{2}}
            {n^\frac{1}{2}}
            +
            \left(
                \frac{n}{\epsilon T}
            \right)^\frac{1}{2}
        \right)\\
        &=
        \tilde{O}
        \left(
            1 
            +
            \frac{\epsilon T}{n}
            +
            \left(
                \frac{\epsilon T}{n}
            \right)^\frac{1}{2}
            +
            \left(
                \frac{n}{\epsilon T}
            \right)^\frac{1}{2}
        \right). \numberthis \label{formula: in case 3-2, psi i_0 bound}
    \end{align*}
    In {\bf Case 3}, since we know that $1 \le \epsilon T \le n$,
    we get that $\frac{\epsilon T}{n} \le 1$.
    Therefore, we get the following in \eqref{formula: in case 3-2, psi i_0 bound}.
    \begin{align*}
        \psi(i_0)
        &=
        \tilde{O}
        \left(
            1 
            +
            \frac{\epsilon T}{n}
            +
            \left(
                \frac{\epsilon T}{n}
            \right)^\frac{1}{2}
            +
            \left(
                \frac{n}{\epsilon T}
            \right)^\frac{1}{2}
        \right)\\
        &=
        \tilde{O}
        \left(
            1 
            +
            \left(
                \frac{n}{\epsilon T}
            \right)^\frac{1}{2}
        \right).
    \end{align*}
    By using this result, we get the following.
    \begin{align*}
        \min_{i_0 \le i \le n} \psi(i)
        &\le
        \psi \left( i_0 \right)\\
        &=
        \tilde{O}
        \left(
            1 
            +
            \left(
                \frac{n}{\epsilon T}
            \right)^\frac{1}{2}
        \right). \numberthis \label{formula: in case 3-2, psi bound}
    \end{align*}
    Combining \eqref{formula: in case3-2, def of psi} and \eqref{formula: in case 3-2, psi bound}, the achieved minimum regret over $i_0 \le i \le n$ is the following.
    \begin{align*}
        R_n
        &=
        \tilde{O}
        \left(
            n
            \sqrt{
            1 
            +
            \left(
                \frac{n}{\epsilon T}
            \right)^\frac{1}{2}
            }
        \right)\\
        &= 
        \tilde{O}
        \left(
            n
            \left(
            1 
            +
            \left(
                \frac{n}{\epsilon T}
            \right)^\frac{1}{4}
            \right)
        \right). \numberthis \label{formula: in case 3-2, reg}
    \end{align*}
    As a result, combining the {\bf Case 3-1} \eqref{formula: in case 3-1, reg} and the {\bf Case 3-2} \eqref{formula: in case 3-2, reg}, the achieved minimum regret over $1 \le i \le n$ in {\bf Case 3} is the following.
    \begin{align*}
        R_n &= \tilde{O}
        \left(
            \min \left(
            n^\frac{4}{5}
            T^\frac{1}{5}
            \epsilon^\frac{1}{5}
            ,
            n\left(
            1 
            +
            \left(
                \frac{n}{\epsilon T}
            \right)^\frac{1}{4}
            \right)
            \right)
        \right)\\
        &= \tilde{O}
        \left(
            n^\frac{4}{5}
            T^\frac{1}{5}
            \epsilon^\frac{1}{5}
        \right), \numberthis \label{formula: in case 3, reg}
    \end{align*}
    where the last equality follows from the fact that, since 
    $1 \le \epsilon T \le n$ in {\bf Case 3},
    \begin{align*}
        n^\frac{4}{5}
        T^\frac{1}{5}
        \epsilon^\frac{1}{5}
        &=
        n^\frac{4}{5}
        (\epsilon T)^\frac{1}{5}\\
        &\le
        n^\frac{4}{5}
        n^\frac{1}{5}\\
        &=
        n\\
        &<
        n\left(
            1 
            +
            \left(
                \frac{n}{\epsilon T}
            \right)^\frac{1}{4}
        \right).
    \end{align*}
    This is the end of {\bf Case 3}.
    \\
    {\bf Case 4:}
    In this case, we know that $n < \epsilon T$.
    Therefore, for any $1 \le i \le n$, the inequality of $i > \frac{n}{\epsilon T}$ holds.
    In this case, all $C_{\epsilon, \ct_{i, j}}$ in the sum of \eqref{formula: in sq uniform, reg bound} are equal to $C^2_{\epsilon, i}$.
    We can obtain the same regret bound
    as \eqref{formula: in case3-2, def of psi} with
    $\min_{1 \le i \le n} \psi(i) \le \psi(1)$.
    Recall that the obtained regret bound is as follows.
    \begin{align*}
        R_n
        &=
        \tilde{O}
        \left(
            n \sqrt{
                \frac{1}{i}
                +
                \phi \left(
                    \sqrt{
                        i
                        +
                        \frac{1}{i}
                        \left(
                            \frac{n}{\epsilon T}
                        \right)^2
                        +
                        \frac{\epsilon T}{n}
                    }
                \right)
            }
        \right)\\
        &:= \tilde{O} 
        \left(
            n \sqrt{\psi(i)}
        \right),
    \end{align*}
    where the function $\psi(i)$ is calculated as follows.
    \begin{align*}
        \psi(i)
        &=
        \frac{1}{i}
        +
        \log
                \left(
                    \sqrt{
                        i
                        +
                        \frac{1}{i}
                        \left(
                            \frac{n}{\epsilon T}
                        \right)^2
                        +
                        \frac{\epsilon T}{n}
                    }
                \right)  
                +
                \left(
                    \sqrt{
                        i
                        +
                        \frac{1}{i}
                        \left(
                            \frac{n}{\epsilon T}
                        \right)^2
                        +
                        \frac{\epsilon T}{n}
                    }
                \right)^{-1}.
    \end{align*}
    The $\psi(1)$ is
    \begin{align*}
        \psi(1)
        &=
        \tilde{O}
        \left(
            1
            +
            \log 
            \sqrt{
                1
                +
                \frac{n^2}{\epsilon^2 T^2}
                +
                \frac{\epsilon T}{n}
            }
            +
            \frac{1}{
            \sqrt{
                1
                +
                \frac{n^2}{\epsilon^2 T^2}
                +
                \frac{\epsilon T}{n}
            }
            }
        \right)\\
        &=
        \tilde{O}
        \left(
            1
            +
            \frac{\epsilon T}{n}
            +
            \left(
                \frac{n}{\epsilon T}
            \right)^\frac{1}{2}
        \right)\\
        &=
        \tilde{O}
        \left(
            1
            +
            \frac{\epsilon T}{n}
        \right),
    \end{align*}
    where the last inequality follows from the fact that
    $n < \epsilon T$.
    Therefore, in {\bf Case 4}, the achieved regret upper bound is as follows.
    \begin{align*}
        R_n
        &=
        \tilde{O}
        \left(
            n \sqrt{
            1
            +
            \frac{\epsilon T}{n}
            }
        \right)\\
        &=
        \tilde{O}
        \left(
            n \left(
            1
            +
            \left( \frac{\epsilon T}{n} \right)^\frac{1}{2}
            \right)
        \right). \numberthis \label{formula: in case 4, reg}
    \end{align*}
    This is the end of {\bf Case 4}.
    \\
    Combining all results of 
    {\bf Case 1} \eqref{formula: in case 1, reg}, 
    {\bf Case 2} \eqref{formula: in case 2, reg}, 
    {\bf Case 3} \eqref{formula: in case 3, reg}, 
    and {\bf Case 4} \eqref{formula: in case 4, reg}, we get the following regret upper bound.
    \begin{align*}
        R_n
        =
        \begin{cases}
            \tilde{O}
            \left(
                \sqrt{n}
            \right)
            & (\epsilon T < n^{- \frac{3}{2}}) \\
            \tilde{O}
            \left(
                n^\frac{4}{5}
                T^\frac{1}{5}
                \epsilon^\frac{1}{5}
            \right)
            & (n^{- \frac{3}{2}} \le \epsilon T \le n)\\
            \tilde{O}
            \left(
            n \left(
            1
            +
            \left( \frac{\epsilon T}{n} \right)^\frac{1}{2}
            \right)
            \right)
            & (n < \epsilon T).
        \end{cases}
    \end{align*}
    This completes the proof.
\end{proof}
\begin{lem} \label{lem: sq biased}
    Suppose that the space kernel is the squared exponential kernel and evaluation time $\{ t_i \}_{i=1}^n$ are extremely biased, that is,
    $t_i = 0$ when $i \neq n_0$ and $t_{n_0} = T$.
    Then, 
    \begin{align*}
    R_n = \tilde{O} 
    \left(
        \sqrt{
            n
        }
    \right),
    \end{align*}
    with high probability.
\end{lem}
\begin{proof}
    We consider the following partition of
    $\{ 0, \ldots, n \}$:
    \begin{align*}
        \{ d_i \} = \{
            d_0 = 0,
            d_1 = n_0 - (\alpha + 1),
            d_2 = n_0 + \alpha,
            d_3 = n
        \},
    \end{align*}
    where $\alpha$ is a positive integer to be specified later,
    which is smaller than $n_0 - 1$ and $n - n_0$.
    This partition corresponds to the following finite subsets of $\ct$.
    \begin{align*}
        \ct_1 &= \{ \tau_1, \ldots, \tau_{n_0 - (\alpha + 1)} \}\\
        \ct_2 &= \{ \tau_{n_0 - \alpha}, \ldots, \tau_{n_0 + \alpha} \}\\
        \ct_3 &= \{ \tau_{n_0 + (\alpha + 1)}, \ldots, \tau_n \}.
    \end{align*}
    We set $\ct' = \ct_i$ in Lemma \ref{lem: bia uni} for each $i = 1, 2, 3$.
    Note that only $\ct_2$ contains $\tau_{n_0}$.
    In the extremely biased setting, we get the evaluation time uniformity
    $C_{\epsilon, \ct_i}$ is
    if $i = 1$ or $i = 3$
    \begin{align*}
        C_{\epsilon, \ct_i} = 0, \numberthis \label{formula: in sq biased, c1}
    \end{align*}
    and if $i = 2$
    \begin{align*}
        C_{\epsilon. \ct_i} =
        2
        (\alpha^2 - 1)
        \min \left(
            \frac{1}{\epsilon^2},
            T^2
        \right). \numberthis \label{formula: in sq biased, c2}
    \end{align*}
    Recall that the upper bound of the cumulative regret is given by \eqref{formula: reg upper bound}.
    By substituting \eqref{formula: in sq biased, c1} and \eqref{formula: in sq biased, c2} into \eqref{formula: in sq uniform, reg bound}, we get folloiwngs.
    \begin{align*}
        R_n
        &\le 
        \sqrt{
            C
            \beta_n
            n
            \left(
                N \gamma_M
                +
                \frac{1}{2}
                \sum_{i=1}^{N}
                M_i
                \phi \left(
                    \sigma^{-2}
                    \epsilon
                    \sqrt{
                        \frac{C_{\epsilon, \ct_i}}{M_i}
                    }
                \right)
            \right)
        } + 2\\
        &=
        \sqrt{
            C
            \beta_n
            n
            \left(
                3 \gamma_M
                +
                \frac{1}{2}
                \phi \left(
                    \sigma^{-2}
                    \epsilon
                    \sqrt{
                        \frac{C_{\epsilon, \ct_2}}{2 \alpha + 1}
                    }
                \right)
            \right)
        } + 2\\
        &=
        \sqrt{
            C
            \beta_n
            n
            \left(
                3 \gamma_M
                +
                \frac{1}{2}
                \phi \left(
                    \sigma^{-2}
                    \epsilon
                    \sqrt{
                        \frac{
                        2
                        (\alpha^2 - 1)
                        \min \left(
                            \frac{1}{\epsilon^2},
                            T^2
                        \right)
                        }
                        {2 \alpha + 1}
                    }
                \right)
            \right)
        } + 2. \numberthis \label{formula: in sq biased, reg bound}
    \end{align*}
    For the squared exponential kernel, the maximum space information gain $\gamma_i$ is $\tilde{O}(1)$ \citep{Srinivas_et_al_2010}.
    By substituting this result and $\alpha = 2$ into \eqref{formula: in sq biased, reg bound}
    and simplifying it, we get the following.
    \begin{align*}
        R_n
        &=
        \tilde{O}
        \left(
            \sqrt{
                n
                \left(
                    1
                    +
                    \min(1, \epsilon T)
                \right)
            }
        \right)\\
        &=\tilde{O}(\sqrt{n}).
    \end{align*}
    This completes the proof.
\end{proof}
\begin{lem} \label{lem: matern uniform}
    Suppose that the space kernel is the {\matern} kernel with parameter $\nu$ and evaluation time $\{ t_i \}_{i=1}^n$ are uniform, that is, $t_i = \frac{T}{n}$.
    Let $c = \frac{d(d+1)}{2 \nu + d(d+1)}$.
    Then, 
    if $\epsilon T < n^{-\frac{3}{2}}$, then
    \begin{align*}
        R_n = \tilde{O}(\sqrt{n^{1+c}},
    \end{align*}
    and if $n^{-\frac{3}{2}} \le \epsilon T \le n$, then
    \begin{align*}
        R_n = \tilde{O} \left(n^\frac{4-c}{5-2c} T^\frac{1-c}{5-2c} \epsilon^\frac{1-c}{5-2c} \right),
    \end{align*}
    and if $n < \epsilon T$, then
    \begin{align*}
        R_n = \tilde{O} \left(n \left(
            1
            +
            \left( 
                \frac{\epsilon T}{n}
            \right)^{\frac{1}{2}}
            \right)
        \right),
    \end{align*}
    with high probability.
\end{lem}
\begin{proof}
    The proof is almost same as that of Lemma \ref{lem: sq uniform}.
\end{proof}
\begin{lem} \label{lem: matern biased}
    Suppose that the space kernel is the {\matern} kernel with parameter $\nu$ and evaluation time $\{ t_i \}_{i=1}^n$ are extremely biased, that is,
    $t_i = 0$ when $i \neq n_0$ and $t_{n_0} = T$.
    Let $c = \frac{d(d+1)}{2 \nu + d(d+1)}$.
    Then, 
    \begin{align*}
    R_n = \tilde{O} 
    \left(
        \sqrt{
            n^{1+c}
        }
    \right),
    \end{align*}
    with high probability.
\end{lem}
\begin{proof}
    The proof is almost same as that of Lemma \ref{lem: sq biased}.
\end{proof}

Combining all results, we obtain the result of Theorem \ref{thm: derived}.

\section{Additional Experiments}
In this section, we will show results of additional experiments which use another acquisition function, Expected Improvement (EI) acquisition function \cite{Jones_et_al_1998}.

The all experimental setting are same as those in Section \ref{sec: ex}.
The results are given in Figure \ref{fig: avg reg ei}.
\begin{figure*}[t]c
    \centering
    \begin{tabular}{c}
        \begin{minipage}{0.4\hsize}
            \centering
\includegraphics[width = 0.9\columnwidth]{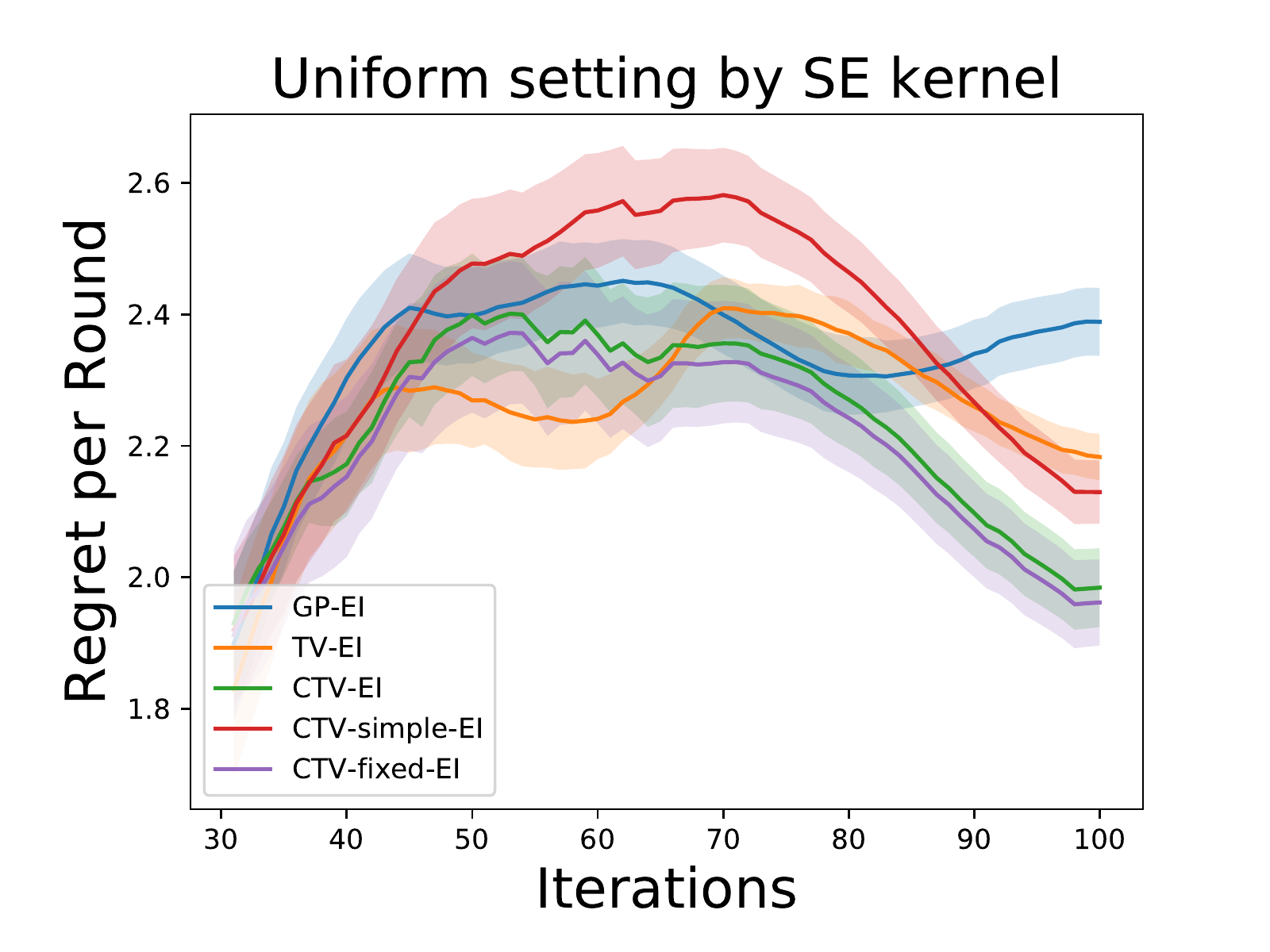}
        \end{minipage}
        \begin{minipage}{0.4\hsize}
        \centering
            \includegraphics[width = 0.9\columnwidth]{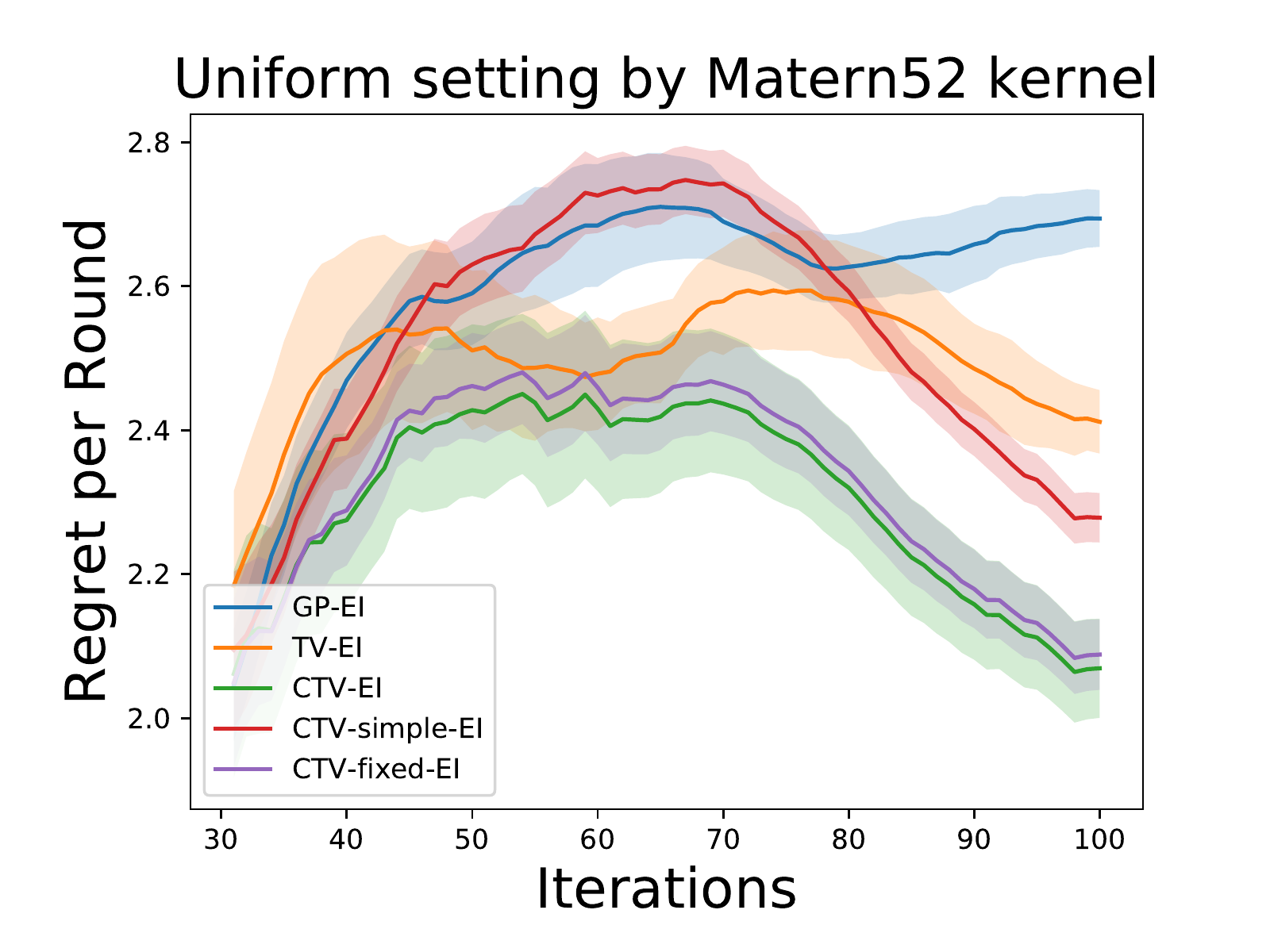}
        \end{minipage} \\
        \begin{minipage}{0.4\hsize}
            \centering
            \includegraphics[width = 0.9\columnwidth]{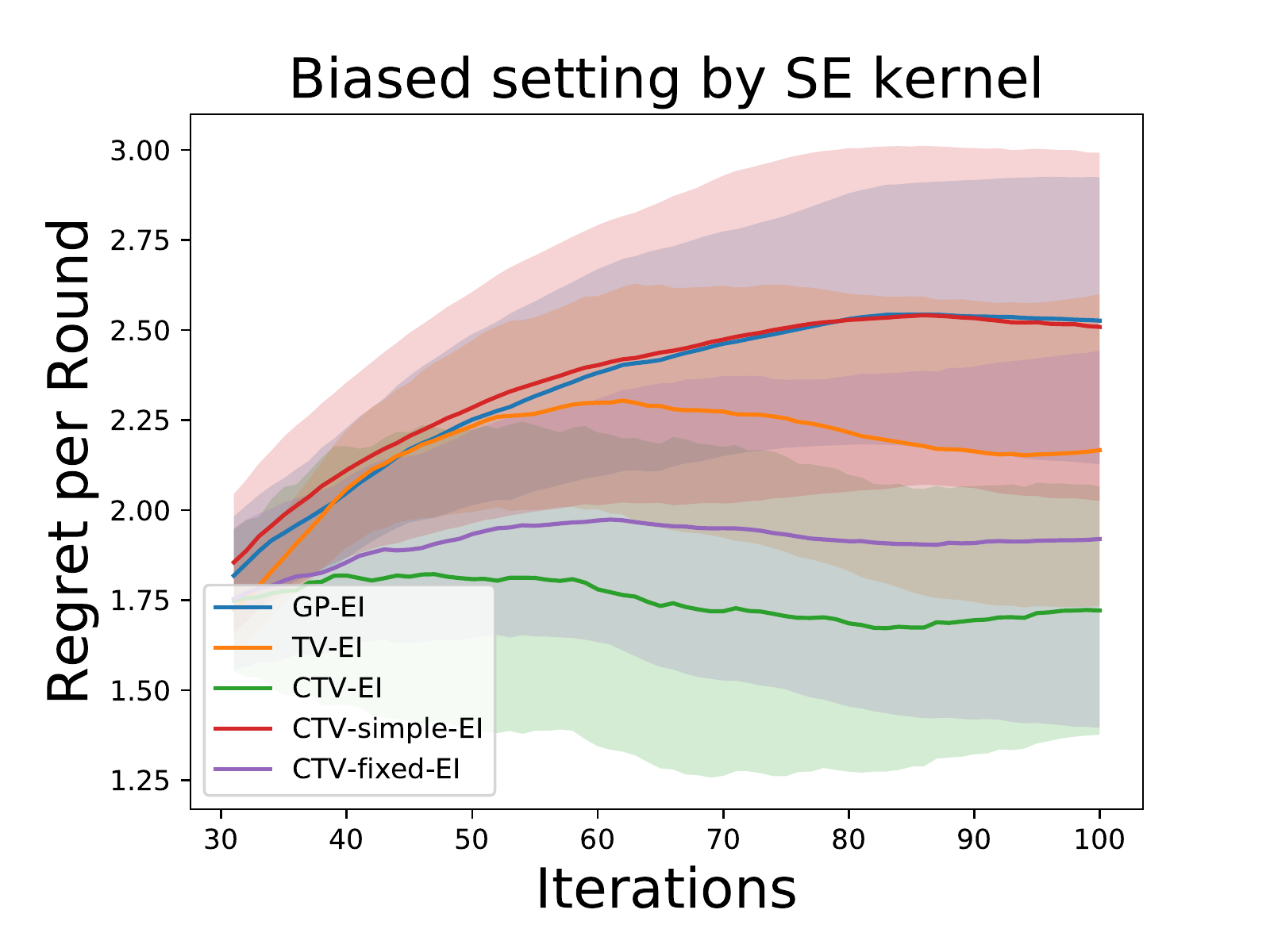}
        \end{minipage}
        \begin{minipage}{0.4\hsize}
        \centering
            \includegraphics[width = 0.9\columnwidth]{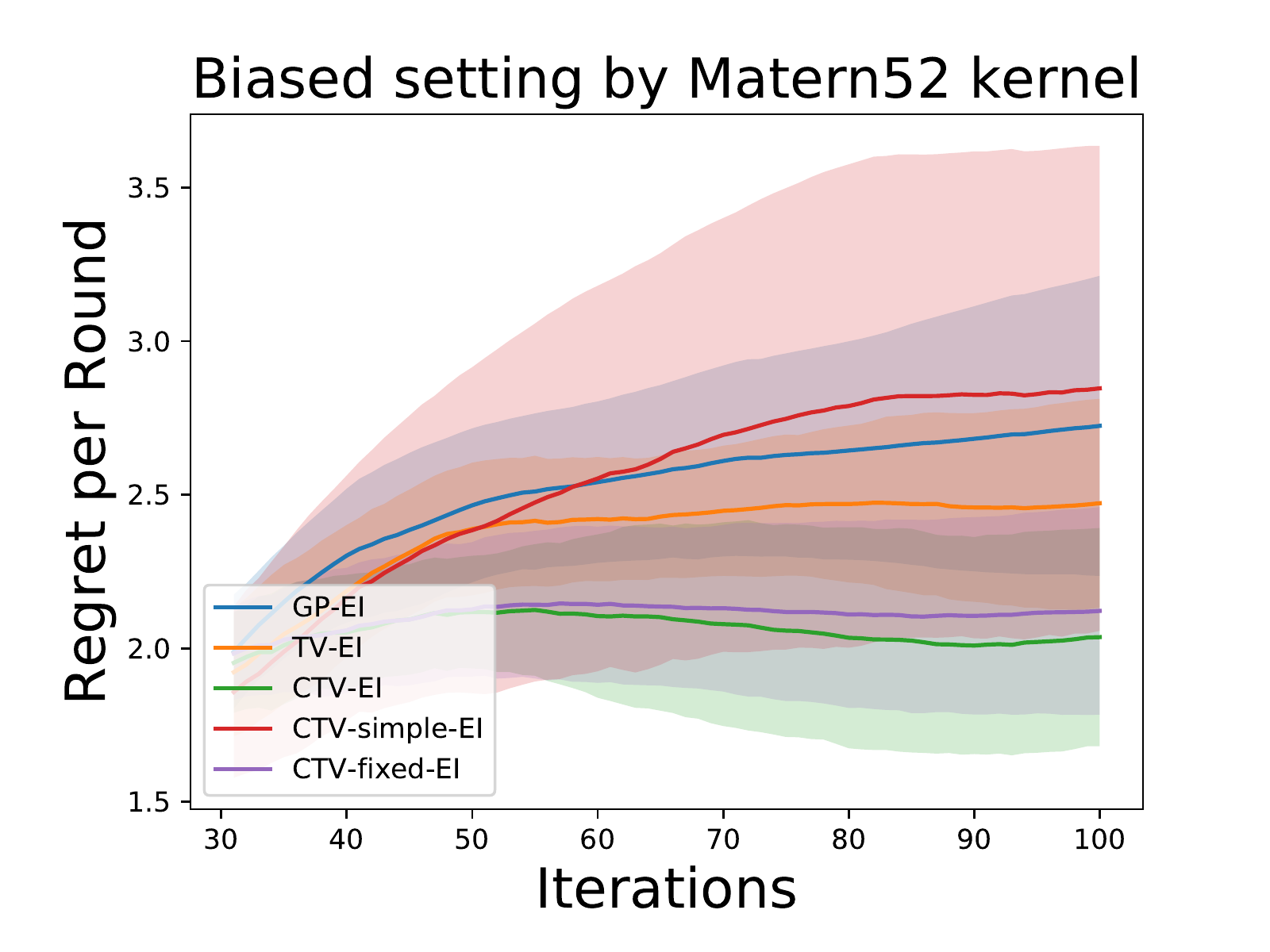}
        \end{minipage}
    \end{tabular}
    \caption{Averaged cumulative regret for the squared exponential and {\matern} (5/2) kernel in the uniform and biased settings.}
    \label{fig: avg reg ei}
\end{figure*}

\end{document}